\documentclass[12pts]{article}
\usepackage[utf8]{inputenc}
\usepackage[round]{natbib}
\usepackage[margin=1in]{geometry}
\usepackage{amsmath}
\usepackage{amssymb}
\usepackage{mathtools}
\newcommand{\vertiii}[1]{{\left\vert\kern-0.25ex\left\vert\kern-0.25ex\left\vert #1 
    \right\vert\kern-0.25ex\right\vert\kern-0.25ex\right\vert}}
\usepackage{amsthm}
\usepackage{hyperref}
\usepackage{xcolor}
\definecolor{RoyalBlue}{rgb}{0, 0, 179}
\hypersetup{
    citecolor=RoyalBlue,
    colorlinks=true,
    linkcolor=RoyalBlue,
    filecolor=magenta,      
    urlcolor=cyan,
}
\usepackage{thmtools}
\usepackage{thm-restate}
\usepackage{cleveref}
\usepackage[mathscr]{euscript}
\usepackage{bbm}
\usepackage{mathtools}
\usepackage{mathabx}
\usepackage{graphicx}
\usepackage{subfigure}
\usepackage{tikz}
\usetikzlibrary{positioning}
\usetikzlibrary{decorations.markings}
\usepackage{array}
\usepackage{pifont}
\usepackage{algorithm}
\usepackage{algpseudocode}
\def\Real{\mathop{\mathbb{R}}\nolimits}

\def\argmin{\mathop{\rm argmin}}

\newcommand{\prob}{\mathbb{P}}
\newcommand{\one}{\mathbbm{1}}

\newcommand{\bs}{\boldsymbol{s}}

\newcommand{\bv}{\boldsymbol{v}}

\newcommand{\bx}{\boldsymbol{x}}

\newcommand{\bsigma}{\boldsymbol{\sigma}}

\newcommand{\cB}{ \mathcal{B}}
\newcommand{\cC}{ \mathcal{C}}

\newcommand{\cF}{ \mathcal{F}}

\newcommand{\cH}{ \mathcal{H}}

\newcommand{\cL}{ \mathcal{L}}

\newcommand{\cN}{ \mathcal{N}}
\newcommand{\cO}{ \mathcal{O}}

\newcommand{\cR}{ \mathcal{R}}
\newcommand{\cS}{ \mathcal{S}}

\newcommand{\cW}{ \mathcal{W}}
\newcommand{\cX}{ \mathcal{X}}
\newcommand{\cY}{ \mathcal{Y}}

\newcommand{\sF}{ \mathscr{F}}
\newcommand{\sG}{ \mathscr{G}}
\newcommand{\sH}{ \mathscr{H}}

\newcommand{\sN}{ \mathscr{N}}

\newcommand{\sR}{ \mathscr{R}}

\newcommand{\E}{\mathbb{E}}

\newcommand{\fL}{\mathbb{L}}

\newcommand{\lp}{{\hat{\lambda}^p_m}}
\usepackage{stackengine}
\theoremstyle{plain}
\theoremstyle{definition}
\newtheorem{definition}{Definition}

\newtheorem{lemma}[definition]{Lemma}
\newtheorem{corollary}[definition]{Corollary}
\newtheorem{assumption}{Assumption}
\theoremstyle{remark}

\DeclareMathOperator{\esssup}{ess\,sup}
\usepackage[max2]{authblk}
\title{A Statistical Analysis of Deep Federated Learning for Intrinsically Low-dimensional Data}

\author[1]{Saptarshi Chakraborty\thanks{Email: saptarsc@umich.edu}}
 \author[2,3,4]{Peter L. Bartlett\thanks{Email: peter@berkeley.edu}}
 \affil[1]{Department of Statistics, University of Michigan}
 \affil[2]{Department of Statistics, University of California, Berkeley}
  \affil[3]{Department of Electrical Engineering and Computer Sciences, University of California, Berkeley}
 \affil[4]{Google DeepMind}
\date{\vspace{-5ex}}

\begin{document}

\maketitle

\begin{abstract}
Federated Learning (FL) has become a revolutionary paradigm in collaborative machine learning, placing a strong emphasis on decentralized model training to effectively tackle concerns related to data privacy. Despite significant research on the optimization aspects of federated learning, the exploration of generalization error, especially in the realm of heterogeneous federated learning, remains an area that has been insufficiently investigated, primarily limited to developments in the parametric regime. This paper delves into the generalization properties of deep federated regression within a two-stage sampling model. Our findings reveal that the intrinsic dimension, characterized by the entropic dimension, plays a pivotal role in determining the convergence rates for deep learners when appropriately chosen network sizes are employed. Specifically, when the true relationship between the response and explanatory variables is described by a $\beta$-H\"older function and one has access to $n$ independent and identically distributed (i.i.d.) samples from $m$ participating clients, for participating clients, the error rate scales at most as  $\Tilde{\cO}\left((mn)^{-2\beta/(2\beta + \bar{d}_{2\beta}(\lambda))}\right)$, whereas for non-participating clients, it scales as  $\Tilde{\cO}\left(\Delta \cdot m^{-2\beta/(2\beta + \bar{d}_{2\beta}(\lambda))} + (mn)^{-2\beta/(2\beta + \bar{d}_{2\beta}(\lambda))}\right)$. Here $\bar{d}_{2\beta}(\lambda)$ denotes the corresponding $2\beta$-entropic dimension of $\lambda$, the marginal distribution of the explanatory variables. The dependence between the two stages of the sampling scheme is characterized by $\Delta$. Consequently, our findings not only explicitly incorporate the ``heterogeneity" of the clients,  but also highlight that the convergence rates of errors of deep federated learners are not contingent on the nominal high dimensionality of the data but rather on its intrinsic dimension.
\end{abstract}

\section{Introduction}
 Federated Learning (FL) stands at the forefront of collaborative machine learning techniques, revolutionizing data privacy and model decentralization in the digital landscape. This innovative approach, first introduced by Google in 2016 through its application in Gboard, has garnered substantial attention and research interest due to its potential to train machine learning models across distributed devices while preserving data privacy and security \citep{mcmahan2017communication}.  By enabling training on decentralized data sources such as mobile devices, FL addresses privacy concerns inherent in centralized model training paradigms \citep{zhang2021survey}. This transformative framework \citep{10197242,10130784,li2020federated,zhu2024federated,li2021model} allows devices to collaboratively learn a shared model while keeping sensitive information local, presenting a promising path forward for advancing technologies in a privacy-preserving manner.  This approach is particularly valuable in domains such as medical imaging, e-health, etc. where sensitive data cannot be easily centralized \citep{yu2024communication,guan2024federated}. Recent advances in FL have addressed key limitations, including communication bottlenecks and heterogeneity across clients. For instance, \citet{jiang2022model} proposed PruneFL, which leverages adaptive parameter pruning to reduce the communication and computational burden on edge devices without sacrificing model performance. \citet{doi:10.1137/23M1553820} leverages variance reduction techniques for better efficiency even under non-convex settings.  \citet{ezzeldin2023fairfed} developed a fairness-aware algorithm by allowing a flexible integration with client-side debiasing methods to achieve superior fairness and accuracy, especially under data heterogeneity and real-world distribution shifts.

From a theoretical standpoint, researchers have delved into comprehending the fundamental properties of Federated Learning (FL), mainly focusing on its optimization characteristics. \textcolor{black}{For instance,  \citet{dinh2022new} analyze federated multitask learning with Laplacian regularization and derive convergence rates that depend on task similarity, offering a principled framework for capturing inter-client relationships. \citet{sattler2019robust,sattler2020clustered} propose robust and clustered FL approaches, respectively, both accompanied by rigorous convergence guarantees under non-i.i.d. settings.} While a substantial body of existing experimental and theoretical work centers on the convergence of optimization across training datasets  \citep{li2020federated,karimireddy2020scaffold,mitra2021linear,mishchenko2022proxskip,yun2022minibatch}, the exploration of generalization error, a crucial aspect in machine learning, appears to have received less meticulous scrutiny within the domain of heterogeneous federated learning. The existing research on the generalization error of FL  primarily focuses on actively participating clients \citep{mohri2019agnostic,qu2022generalized}, neglecting the disparities between these observed data distributions from actively participating clients and the unobserved distributions inherent in passively nonparticipating clients. In practical federated settings, a multitude of factors, such as network reliability or client availability, influence the likelihood of a client's participation in the training process. Consequently, the actual participation rate may be small, leading to a scenario where numerous clients never partake in the training phase \citep{kairouz2021advances,yuan2021we}.

 From a generalization viewpoint, when one has access to $m$ participating clients and each client generates $n$ many i.i.d. observations, \citet{mohri2019agnostic} showed a rate of $\cO(1/\sqrt{mn})$ holds for the excess risk of the participating clients.   \citet{chen2020distribution} improved upon this risk bound by deriving fast rates of $\cO(1/mn)$ for the expected excess risk for bounded losses.  Recently,  \citet{hu2023generalization} derived the generalization bounds for participating and nonparticipating clients under a generic unbounded loss under different model assumptions such as small-ball property or sub-Weibullness of the underlying distributions. 
 
 Despite the growing interest on the problem, the current literature suffers in many aspects. Firstly, the generalization bounds derived in the literature are parametric in nature and overlook the misspecification error inherent in the model assumptions. Secondly, the derived bounds do not take into account the size of the networks used and its impact on the generalization performance. Furthermore, one key aspect ignored by the current framework is that the data on which these models are trained are typically intrinsically low-dimensional in nature \citep{pope2020intrinsic}. However, the current bounds do not explore the dependence  the intrinsic dimension of the data, questioning the efficacy of the current theoretical understanding concerning the real-world complexities of the problem.

The recent theoretical developments in the generalization aspects of deep learning theory literature have revealed that the excess risk for different deep learning models, especially regression \citep{schmidt2020nonparametric, suzuki2018adaptivity} and generative models \citep{JMLR:v23:21-0732,chakraborty2024a} exhibit a decay pattern that depends only on the intrinsic dimension of the data. Notably,  \citet{JMLR:v21:20-002} and  \citet{JMLR:v23:21-0732} showed that the excess risk decays as $\cO(n^{-1/\cO(\overline{\text{dim}}_M(\mu))})$, where $\overline{\text{dim}}_M(\mu)$ denotes the Minkowski dimension of the underlying distribution (see Section~\ref{int} for a detailed overview). Nevertheless, it is important to highlight that the Minkowski dimension primarily focuses on measuring the growth rate in the covering number of the \textit{entire} support, overlooking scenarios where the distribution may have  higher concentrations of mass within a specific sub-regions. As a result, the Minkowski dimension often overestimates the intrinsic dimension of the data distribution, leading to slower rates of statistical convergence. In contrast, some studies \citep{chen2022nonparametric,chen2019efficient,jiao2021deep,dahal2022deep} attempt to impose a smooth Riemannian manifold structure on this support and characterize the rate through the dimension of this manifold. However, this assumption is not only very strong and unverifiable in practical terms, but  also ignores the possibility that the data may be concentrated only in certain sub-regions and thinly spread over the rest, again resulting in an overestimate.

Recent insights from the optimal transport literature introduce the Wasserstein dimension \citep{weed2019sharp}, overcoming these limitations and providing a more accurate characterization of convergence rates when estimating a distribution through the empirical measure. Furthermore, advancements in this field introduce the entropic dimension \citep{JMLR:v26:24-0054}, building upon Dudley's seminal work \citep{dudley1969speed}, and can be applied to describe the convergence rates for Bidirectional Generative Adversarial Networks (BiGANs) \citep{donahue2017adversarial}. Remarkably, this entropic dimension is no larger than the Wasserstein and Minkowski dimensions, resulting in faster rates of convergence for the sample estimator. However, there has been no developments in incorporating these fast rates available in the current deep learning theory literature for federated learning possibly due to the complicated heterogeneity between the clients and inter-dependencies among the data points generated by individual clients.

To address the gap between the theory and practice of federated learning as highlighted above, our work presents a comprehensive examination of the generalization error in a regression context, employing a two-level framework that effectively addresses the overlooked gaps within the current literature. This framework uniquely encapsulates both the diversity and interrelationships present among clients' distributions as well as addresses the misspecification error absent in the present literature. Furthermore, we characterize the low-dimensional nature of the data distribution through the entropic dimension, which is more efficient compared to the Minkowski dimension, which is the benchmark in the deep learning theory literature \citep{JMLR:v21:20-002,JMLR:v23:21-0732}. Our utilization of the entropic dimension results in superior bounds, surpassing those derived from other dimensions like the Minkowski and Wasserstein dimensions. 

\paragraph{Contributions} The main contributions of this paper can be summarized as follows:
\begin{itemize}
    \item We study the generalization properties of deep federated learning in a two-stage Bayesian sampling setting when the relation between the response and explanatory variables can be expressed through a $\beta$-H\"older function and an additive sub-Gaussian noise.
    \item We show that when one has access to $n$ i.i.d. samples from each of the $m$ participating clients, the excess risk for the participating clients scales as $\Tilde{\cO}\left((mn)^{-2\beta/(2\beta + \bar{d}_{2\beta}(\lambda))} \right)$, where $\bar{d}_{2\beta}(\lambda)$ denotes the $2\beta$-entropic dimension of $\lambda$, the marginal distribution of the explanatory variables.
    \item Furthermore, for nonparticipating clients, the error rate scales as \[\cO\left( \Delta \cdot m^{-(2\beta/\bar{d}_{2\beta}(\lambda)+2\beta)} + (mn)^{-(2\beta/\bar{d}_{2\beta}(\lambda)+2\beta)}\right),\] primarily depending on the number of participating clients, $m$, when a large amount of data is available for each participating client. Here $\Delta = \min\left\{\|\operatorname{KL}(\lambda_\theta, \lambda)\|_{\psi_1}, 1\right\}$ characterizes the heterogeneity of the client's distribution in terms of the Orlicz-1 norm of the discrepancy in terms of the $\operatorname{KL}$-divergence.
\end{itemize}
The proposed analyses not only provides a concise framework to understand the error rate for the deep federated regression problem but the main theorems (Theorems~\ref{mainthm_2} and~\ref{mainthm}) also yield constructive guidance on architectural design, characterizing how the required network size should scale with fundamental problem parameters — including the number of samples from each client ($n$), the number of participating clients ($m$), the intrinsic dimension of the data measured via the entropic dimension $\bar{d}_{2\beta}(\lambda)$, and the degree of cross-client heterogeneity, $\Delta(\theta, \bx)$. These results not only extend classical statistical insights to federated settings but also illuminate the interplay between expressivity, data geometry, and client variability in shaping optimal learning performance.

\paragraph{Organization} The remainder of the paper is organized as follows. In Section~\ref{background}, we introduce the necessary notations and background. In Section~\ref{ps} we introduce the problem at hand, followed by a simulation study in Section~\ref{simu} to understand the effect of the intrinsic dimension on the error bounds. In Section~\ref{tana}, we discuss the main theoretical results of the paper along with the necessary assumptions. We then give a brief proof overview of the main results in Section~\ref{pf}, followed by concluding remarks and discussions in Section~\ref{con}.

\section{Background} \label{background}
\subsection{Notations and Definitions}
This section recalls some of the notations and background necessary for our theoretical analyses.  We say $A \precsim B$ (for $A,\,B \ge 0$) if there exists a constant $C>0$, independent of $n$, such that $A \le C B$. 
 For a function $f: \cS \to \Real$ (with $\cS$ being Polish) and a probability measure $\nu$ on $\cS$,  \(\esssup^\nu_{x \in \cS} f(x) = \inf \left\{a: \nu\left( f^{-1}\left((a,\infty)\right)\right)=0\right\}.\)
For any function $f: \cS \to \Real$, and any measure $\gamma$ on $\cS$, let $\|f\|_{\fL_p(\gamma)} : = \left(\int_\cS |f(x)|^p d \gamma(x) \right)^{1/p}$, if $0<p< \infty$. Also let, $\|f\|_{\fL_\infty(\gamma)} : = \esssup_{x \in \cS}^\gamma|f(x)|$. We say $A_n = \tilde{\cO}(B_n)$ if $A_n \le B_n \times \operatorname{polylog}(n)$, for some factor $\operatorname{polylog}(n)$ that is a polynomial in $\log n$.
\begin{definition}[Covering Number] 
    \normalfont 
    For a metric space $(S,\varrho)$, the $\epsilon$-covering number w.r.t. $\varrho$ is defined as:
    \(\cN(\epsilon; S, \varrho) = \inf\{n \in \mathbb{N}: \exists \, x_1, \dots x_n \text{ such that } \cup_{i=1}^nB_\varrho(x_i, \epsilon) \supseteq S\}.\)
\end{definition}

\begin{definition}[Neural networks]\normalfont
 Let $L \in \mathbb{N}$ and $ \{N_i\}_{i \in [L]} \in \mathbb{N}$. Then a $L$-layer neural network $f: \Real^d \to \Real^{N_L}$ is defined as,
\begin{equation}
\label{ee1}
f(x) = A_L \circ \sigma_{L-1} \circ A_{L-1} \circ \dots \circ \sigma_1 \circ A_1 (x)    
\end{equation}
Here, $A_i(y) = W_i y + b_i$, with $W_i \in \Real^{N_{i} \times N_{i-1}}$ and $b_i \in \Real^{N_{i-1}}$, with $N_0 = d$. Note that $\sigma_j$ is applied component-wise.  Here, $\{W_i\}_{1 \le i \le L}$ are known as weights, and $\{b_i\}_{1 \le i \le L}$ are known as biases. $\{\sigma_i\}_{1 \le i \le L-1}$ are known as the activation functions. Without loss of generality, one can take $\sigma_\ell(0) = 0, \, \forall \, \ell \in [L-1]$. We define the following quantities:  
(Depth) $\cL(f) : = L$ is known as the depth of the network; (Number of weights) The number of weights of the network $f$ is denoted as $\cW(f) = \sum_{i=1}^L N_i N_{i-1}$; 
(maximum weight) $\cB(f) = \max_{1 \le j \le L} (\|b_j\|_\infty) \vee \|W_j\|_{\infty}$ to denote the maximum absolute value of the weights and biases.
\begin{align*}
     \cN \cN_{\{\sigma_i\}_{i \in [L-1]}} (L, W, B, R)   = \{ &  f \text{ of the form \eqref{ee1}}: \cL(f) \le L ,  \cW(f) \le W, \\
   & \cB(f) \le B, \sup_{x \in [0,1]^d}\|f(x)\|_\infty \le R  \}.\vspace{-10pt}
\end{align*}
 If $\sigma_j(x) = x \vee 0$, i.e. the ReLU activation, for all $j=1,\dots, L-1$, we use the notation $\cR \cN (L, W, B, R)$ to denote $\cN \cN_{\{\sigma_i\}_{1 \le i \le L-1}} (L, W, B, R)$.
 \end{definition}
 
 \begin{definition}[H\"older functions]\normalfont
Let $f: \mathcal{S} \to \Real$ be a function, where $\mathcal{S} \subseteq \Real^d$. For a multi-index $\bs = (s_1,\dots,s_d)$, let, $\partial^{\bs} f = \frac{\partial^{|\bs|} f}{\partial x_1^{s_1} \dots \partial x_d^{s_d}}$, where, $|\bs| = \sum_{\ell = 1}^d s_\ell $. We say that a function $f: \cS \to \Real$ is $\beta$-H\"{o}lder (for $\beta >0$) if \[\|f\|_{\sH^\beta}:  =  \sum_{\bs: 0 \le |\bs| \le \lfloor \beta \rfloor} \|\partial^{\bs} f\|_\infty   + \sum_{\bs: |\bs| = \lfloor \beta \rfloor} \sup_{x \neq y}\frac{\|\partial^{\bs} f(x)  - \partial^{\bs} f(y)\|_\infty}{\|x - y\|_\infty^{\beta - \lfloor \beta \rfloor}} < \infty.\]
If $f: \Real^{d_1} \to \Real^{d_2}$, then we define $\|f\|_{\sH^{\beta}} = \sum_{j = 1}^{d_2}\|f_j\|_{\sH^{\beta}}$. For notational simplicity, let, $\sH^\beta(\cS_1, \cS_2,C) = \{f: \cS_1 \to \cS_2: \|f\|_{\sH^\beta} \le C\}$. Here, both $\cS_1$ and $\cS_2$ are both subsets of real vector spaces. 
\end{definition}
\begin{definition}[$\operatorname{KL}$-divergence]
     Suppose that $P$ and $Q$ are distributions on $[0,1]^d$. Then,
     \[
         \operatorname{KL}(P,Q) = \begin{cases}
             \int \log \frac{dP}{dQ} dP & \text{ if } P\ll Q\\
             \infty & \text{ Otherwise.}
         \end{cases}
     \]
 \end{definition}

 \begin{definition}[Orlicz norm, \cite{vershynin2018high}] For a random variable $X$, the $\psi_p$-Orlicz norm is defined as: 
     \(\|X\|_{\psi_p} = \inf\{t > 0: \E \exp(|X^p|/t^p) \le 2\}.\)
 \end{definition}
\subsection{Intrinsic Dimension}\label{int}
It is hypothesized that real-world data, particularly vision data, is mostly constrained within a lower-dimensional structure embedded in a high-dimensional feature space \citep{pope2020intrinsic}. To quantify this reduced dimensionality, researchers have introduced various metrics to gauge the effective dimension of the underlying probability distribution that generates the data. Among these methods, the most commonly employed ones involve assessing the rate of growth of the covering number, on a logarithmic scale, for the majority of the support of this data distribution.

Let us consider a compact Polish space \citep{villani2021topics} denoted as $(\cS, \varrho)$, where $\mu$ represents a probability measure defined on it. For the rest of this paper, we will assume that $\varrho$ corresponds to the $\ell_\infty$-norm. The most straightforward measure of the dimension of a probability distribution is the upper Minkowski dimension \citep{falconer2004fractal} of its support, and it is defined as follows:
\[\overline{\text{dim}}_M(\mu) = \limsup_{\epsilon \downarrow 0} \frac{\log\cN(\epsilon;\text{supp}(\mu), \ell_\infty)}{\log (1/ \epsilon)}.\] 
This concept of dimensionality relies solely on the covering number of the support and does not presume the existence of a smooth mapping to a lower-dimensional Euclidean space. As a result, it encompasses not only smooth Riemannian manifolds but also highly non-smooth sets such as fractals. The statistical convergence properties of various estimators related to the upper Minkowski dimension have been extensively investigated in the literature.  \citet{kolmogorov1961} conducted a comprehensive study on how the covering number of different function classes depends on the upper Minkowski dimension of the support. More recently, studies by  \cite{JMLR:v21:20-002},  \cite{JMLR:v23:21-0732} and \cite{chakraborty2024a} demonstrated how deep learning models can leverage this inherent low-dimensionality in data, which is also reflected in their convergence rates. However, a notable limitation associated with utilizing the upper Minkowski dimension is that when a probability measure covers the entire sample space but is concentrated predominantly in specific regions, it may yield a high dimensionality estimate that might not accurately reflect the underlying 

To address the aforementioned challenge, in terms of the intrinsic dimension of a measure $\mu$,  \citet{JMLR:v26:24-0054} introduced the concept of the $\alpha$-entropic dimension of a measure. Before we proceed, we recall the $(\epsilon, \tau)$-cover of a measure 
\citep{posner1967epsilon} as: \[\sN_\epsilon(\mu, \tau) = \inf\{\cN(\epsilon; S, \varrho): \mu(S) \ge 1-\tau\},\]  i.e. $\sN_\epsilon(\mu, \tau)$ counts the minimum number of $\epsilon$-balls required to cover a set $S$ of probability at least $1-\tau$, under the probability measure $\mu$.
 \begin{definition}[Entropic Dimension, \cite{JMLR:v26:24-0054}]\label{ed}\normalfont
    For any $\alpha>0$, we define the $\alpha$-entropic dimension of $\mu$ as:
     \[\bar{d}_\alpha(\mu) = \limsup_{\epsilon \downarrow 0} \frac{\log \sN_\epsilon(\mu,\epsilon^\alpha)}{\log (1/\epsilon)}.\]
\end{definition}
 The $\alpha$-entropic dimension extends Dudley's entropic dimension \citep{dudley1969speed} to characterize the convergence rate for the Bidirectional Generative Adversarial Network (GAN) problem \citep{donahue2017adversarial}. It has been demonstrated that the entropic dimension is no larger than the upper Minkowski dimension and the upper Wasserstein dimension \citep{weed2019sharp}. Moreover, strict inequality holds even for simple examples. For example, for the Pareto distribution, $p_\gamma(x) = \gamma x^{-(\gamma+1)} \mathbbm{1}(x\ge 1)$, it is easy to check that $\overline{\operatorname{dim}}_M(p_\gamma) = \infty$ and $\bar{d}_\alpha(p_\gamma) = 1+\alpha/\gamma$. For a more in-depth exploration, we direct the reader to Section~3 of \citet{JMLR:v26:24-0054}. The study indicated that the entropic dimension serves as a more efficient means of characterizing the intrinsic dimension of data distributions compared to popular measures such as the upper Minkowski dimension or the Wasserstein dimension and enables the derivation of faster rates of convergence for the estimates.
\section{Problem Setup}\label{ps}
We let $\cX = [0,1]^d$ be the data space and $\cY = \Real$ be the outcome space. We assume that there are $m$ clients and each client gives rise to $n$ data points. To conceptualize the two-stage sampling framework in a Bayesian setting, we introduce an unobserved hyper-parameter $\theta$, lying in some parameter space $\Theta$, which we assume to be Polish and compact. This $\theta$ is used to represent a client's inner state: $\theta_i$ represents the state of the $i$-th participating client and we assume that $\theta_1, \ldots, \theta_m$ are independent and identically distributed (i.i.d.) according to the distribution $\pi(\cdot)$ on $\Theta$. $(\bx_{ij},y_{ij})$ denotes the $j$-th sample for the $i$-th participating client. Conditioned on $\theta_i$, we assume that $\{\bx_{ij}\}_{i=1}^n$ are i.i.d. $\lambda_{\theta_i}(\cdot)$. Furthermore, we suppose that the true regression function is $f_0(\cdot)$ and $y_{ij} = f_0(\bx_{ij}) + \epsilon_{ij} $, for zero-mean sub-Gaussian random variables $\epsilon_{ij}$'s which are i.i.d. and are independent of $\theta_i$'s and $\bx_{ij}'s$. To write more succinctly,
\begin{equation}
    \theta_1, \ldots, \theta_m  \overset{i.i.d.}{\sim} \pi(\cdot) ; \hspace{1cm} 
    \bx_{i1}, \ldots, \bx_{in}|\theta_i  \overset{i.i.d.}{\sim} \lambda_{\theta_i}(\cdot) ; \hspace{1cm}
    y_{ij}  = f_0(\bx_{ij}) + \epsilon_{ij},\quad  \epsilon_{ij} \overset{i.i.d.}{\sim} \tau.\label{model}
\end{equation}
The law of $\epsilon_{ij}$'s are denoted as $\tau(\cdot)$ for notational simplicity. \textcolor{black}{The data generation process can be represented through a graphical model as shown in Figure \ref{fig:data_gen}.}

\begin{figure}[t]
\centering
\begin{tikzpicture}[
roundnode/.style={circle, draw=blue!80, fill=blue!5, very thick, minimum size=10mm},
ynode/.style={circle, draw=yellow!80, fill=yellow!5, very thick, minimum size=10mm},
rs/.style={rectangle, draw=red!60, fill=red!5, very thick, minimum size=10mm},
bs/.style={rectangle, draw=blue!60, fill=blue!5, very thick, minimum size=10mm},
bc/.style={circle, draw=blue!60, fill=blue!5, very thick, minimum size=10mm},
rc/.style={circle, draw=red!60, fill=red!5, very thick, minimum size=10mm}
]
\node[bc, label=above:{Client parameter}]      (x)          {$\theta$};
\node[bc, label=below:{Explanatory variable}]      (x1)    [right = of x]  {$\bx$};
\node[rs, label=right:{Response}]    (x2)    [right = of x1] {$y$};
\node[bc, label=above:{Additive Sub-Gaussian noise}]    (x3)    [above = of x2] {$\varepsilon$};
\path[->,thick] (x) edge (x1);
\path[->,thick] (x1) edge node[midway, above] {$f$} (x2);
\path[->,thick] (x3) edge (x2);
\end{tikzpicture}
    \caption{\textcolor{black}{A pictorial representation of the data generation process. The client hyperparameter follows  $\theta \sim \pi$ and the explanatory variable follows $\bx|\theta \sim \lambda_\theta$. The final response is $y = f(x) + \epsilon$, with $f \in \sH^\beta$.}}
\label{fig:data_gen}
\end{figure}

Note that a similar two-level framework was also used by \citet{mohri2019agnostic,chen2021theorem,hu2023generalization}, although under a different model. We posit that this assumption holds practical merit, e.g. cross-device federated learning, where the total number of clients is typically large, and it is reasonable to presume that the $m$ participating clients are selected at random from the pool \citep{reisizadeh2020robust,wang2021field}. In this learning scenario,  the training process solely engages with the $m$ distributions $\{\lambda_{\theta_i}\}_{i=1}^m$, where as, the total number of clients and the count of non-participating clients generally far exceeds  $m$ \citep{xu2020client,yang2020age}. In practical terms, this two-level sampling framework not only captures the diversity among clients' distributions but also underscores the interdependence among these distributions. A similar framework has been employed in recent literature \citep{li2020federated,yuan2021we,wang2021field,hu2023generalization}.

Throughout the remainder of the analysis, we take the loss function as the squared error loss, which emerges as a natural choice for additive noise models.  In practice, one has only access to $\{(\bx_{ij},y_{ij})\}_{i \in [m], \, j \in [n]} $ and obtains an estimate for $f_0$ under the squared error loss as:
\begin{equation}\label{erm}
    \hat{f} = \argmin_{f\in \cF} \sum_{i=1}^m \sum_{j=1}^n (y_{ij} - f(\bx_{ij}))^2.
\end{equation}

Here, $\cF$ is a function class, usually realized through neural networks. In this paper, we take $\cF = \cR\cN(L,W,B,R)$ for some choice of the hyper-parameters. 

Under model \eqref{model}, a new data point for a nonparticipating client is generated as $\theta \sim \pi(\cdot), \, \bx|\theta \sim \lambda_\theta(\cdot) $ and $y = f_0(\bx) + \epsilon$, where, $\epsilon \sim \tau(\cdot)$ and is independent of $\theta$ and $\bx$. \(\lambda(\cdot) = \int \lambda_\theta(\cdot) d\pi(\theta)\) denotes the marginal distribution of the explanatory variables for the nonparticipating clients. The excess risk for the nonparticipating clients is denoted as,
\begin{align}
      \E_{\theta \sim \pi} \E_{\bx|\theta \sim \lambda_\theta} \E_{ y|\bx,\theta} \left[(y - \hat{f}(\bx))^2 - (y - f_0(\bx))^2\right] 
   = & \E_{\theta \sim \pi, \,\bx|\theta \sim \lambda_\theta} \E_{\epsilon \sim \tau} \left[(f_0(\bx) + \epsilon - \hat{f}(\bx))^2 - \epsilon^2\right] \nonumber\\
   = & \E_{\theta \sim \pi} \E_{\bx|\theta \sim \lambda_\theta} (\hat{f}(\bx) - f_0(\bx))^2 \nonumber\\
    = &  \|\hat{f}-f_0\|_{\fL_2(\lambda)}^2. \label{risk1}
\end{align}
Similarly, the marginal distribution for the explanatory variable for participating clients, selected at random is $\lp(\cdot) = \frac{1}{m} \sum_{i=1}^m \lambda_{\theta_i}(\cdot)$. Thus, the excess risk for a participating client, selected at random is given by, 

\begingroup
\allowdisplaybreaks
\begin{align}
    & \E_{\theta \sim \hat{\pi}_m} \E_{\bx|\theta \sim \lambda_\theta} \E_{ y|\bx,\theta} \left((y - \hat{f}(\bx))^2 - (y - f_0(\bx))^2\right) \nonumber\\
    = & \E_{\theta \sim \hat{\pi}_m} \E_{\bx|\theta \sim \lambda_\theta} \E_{\epsilon \sim \tau} \left((f_0(\bx) + \epsilon - \hat{f}(\bx))^2 - \epsilon^2\right) \nonumber\\
    = & \E_{\theta \sim \hat{\pi}_m} \E_{\bx|\theta \sim \lambda_\theta} (\hat{f}(\bx) - f_0(\bx))^2 \nonumber\\
    = &  \|\hat{f}-f_0\|_{\fL_2(\lp)}^2. \label{risk2}
\end{align}
\endgroup
In the above calculations $\hat{\pi}_m \equiv \operatorname{Unif}\left(\{\theta_1,\ldots, \theta_m\}\right)$, denotes the empirical distribution on $\{\theta_1, \ldots, \theta_m\}$.
The goal of this paper is to understand how to choose $\cF$ efficiently to obtain tight bounds on the excess risk in \eqref{risk1} and \eqref{risk2}.
\section{A proof of Concept}\label{simu}
Before delving into the theoretical exploration of the problem, we conduct experiments aimed at demonstrating that the error rates for deep federated regression are primarily contingent on the intrinsic dimension of the data. 
\subsection{Simulations on Synthetic Data}
\begin{figure}[ht]
  \centering
  \subfigure[Participating Clients]{
    \includegraphics[width=0.2\linewidth]{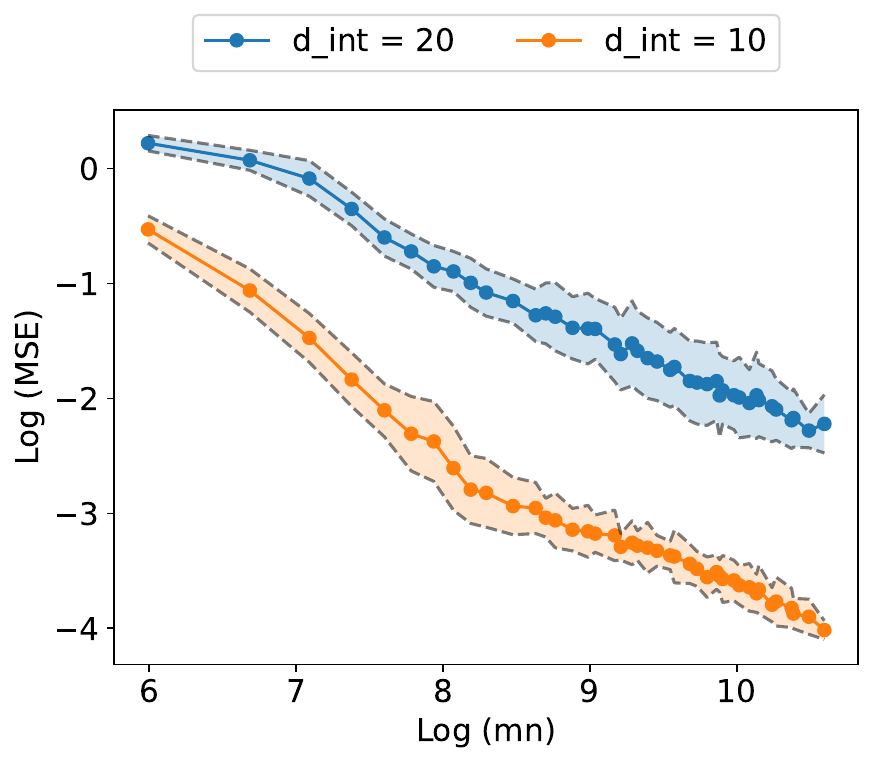}
    \label{fig:11}
  }
  \subfigure[Nonparticipating Clients]{
    \includegraphics[width=0.2\linewidth]{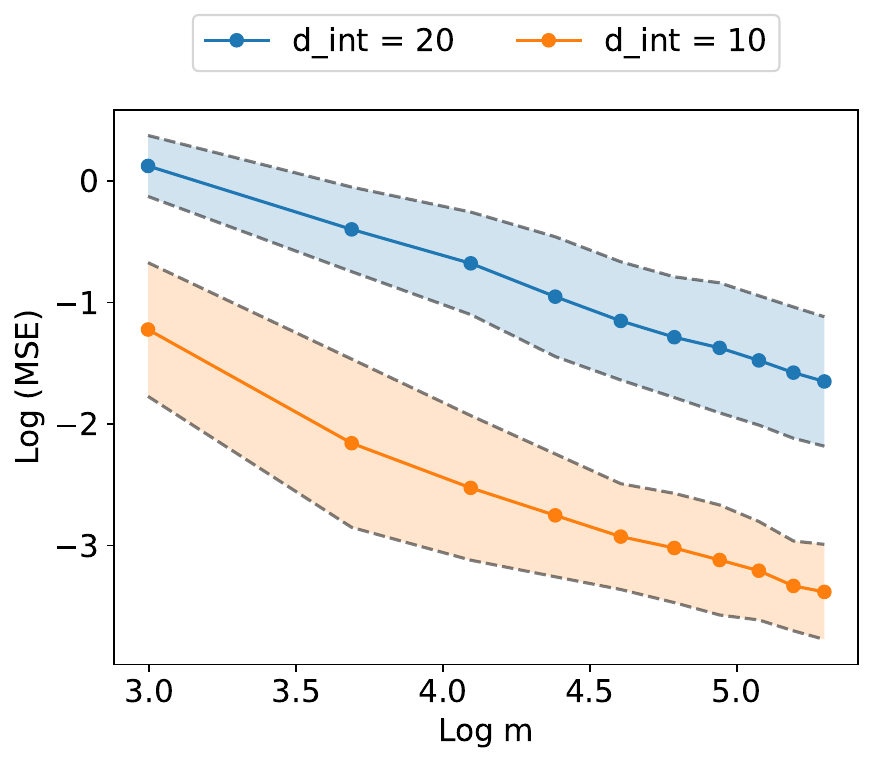}
    \label{fig:12}
  }
\subfigure[Participating Clients]{
    \includegraphics[width=0.2\linewidth]{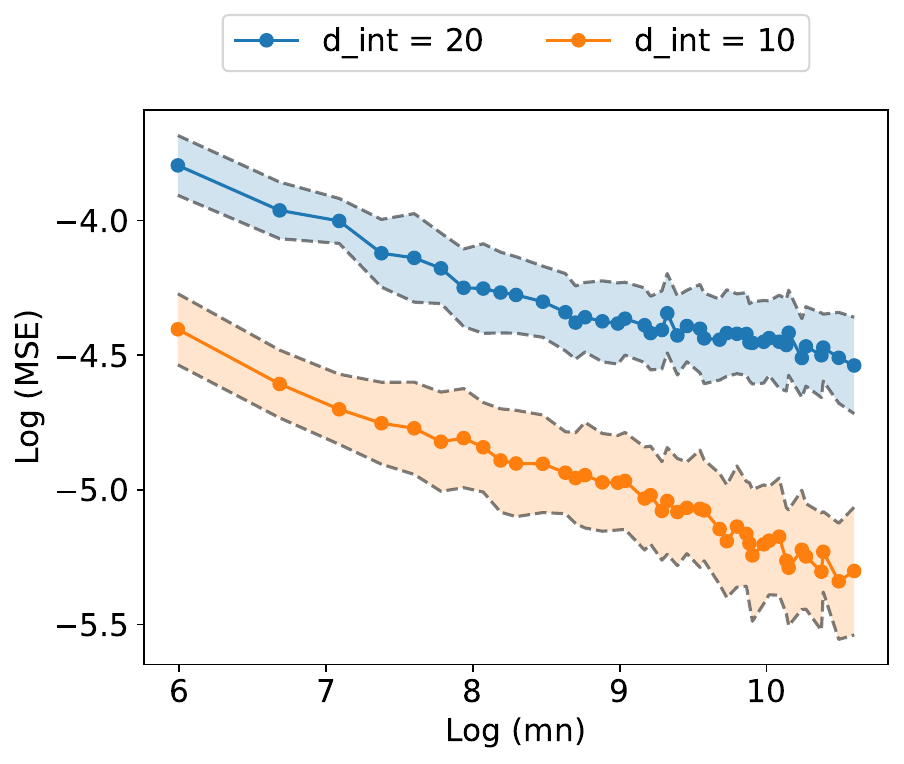}
    \label{fig:21}
  }
  \subfigure[Nonparticipating Clients]{
    \includegraphics[width=0.2\linewidth]{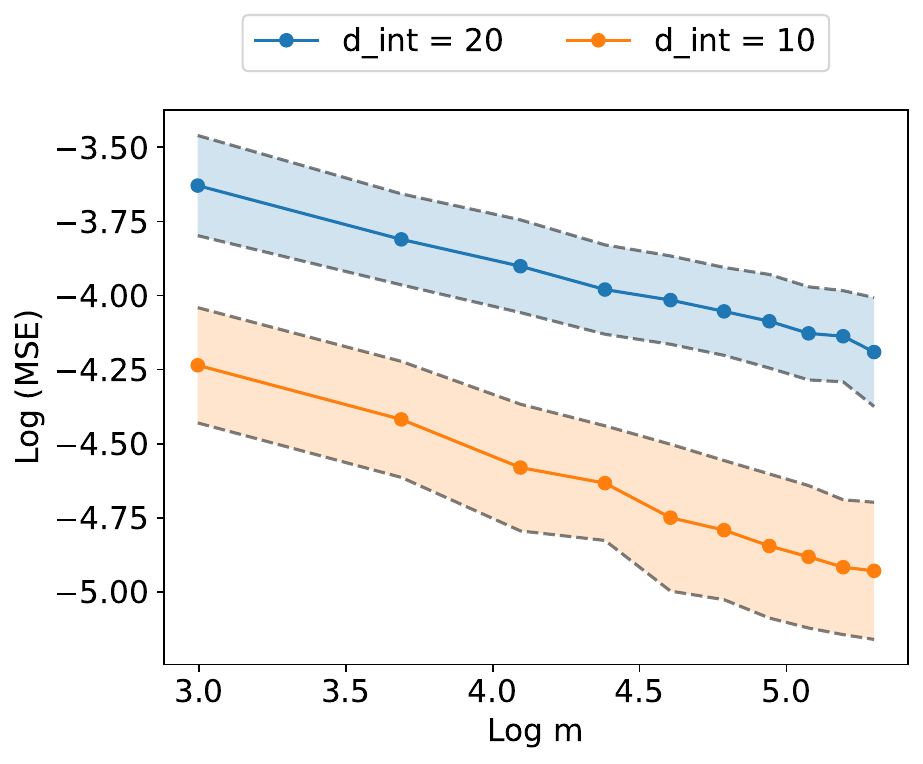}
    \label{fig:22}
  }
  \caption{Average test mean squared error (MSE) is presented for both participating and non-participating clients across two distinct intrinsic dimensions, varying training sample sizes on a logarithmic scale. The error bars and bands illustrate the standard deviation over 20 replications. The top row corresponds to experiments with $f_0^{(1)}$, while the bottom row denotes the performance for $f_0^{(2)}$. Notably, the intrinsic dimensions manifest two distinct decay patterns. As anticipated from theoretical analyses, participating clients exhibit a lower error rate compared to non-participating clients.}
  \label{fig:1}
\end{figure}
We take the true regression function $f_0^{(1)}(\bx) = \frac{1}{d-1} \sum_{i=1}^{d-1} x_ix_{i+1} + \frac{2}{d} \sum_{i=1}^{d}  \sin(2\pi x_i) \one\{x_i \leq 0.5\} + \frac{1}{d} \sum_{i=1}^{d} (4\pi(\sqrt{2} - 1)^{-1} (x_i - 2^{-1/2})^2 - \pi(\sqrt{2} - 1)) \one\{x_i > 0.5\}$. This choice of $f_0$ was used by  \citet{JMLR:v21:20-002}. Clearly, $f_0 \in \sH^2(\Real^d, \Real)$. We take $d=30$ and the first \texttt{d\_int} coordinates of $\bx|\theta$ to be uniformly distributed on the $[\theta,\theta+1]^{\texttt{d\_int}} $. The remaining $d - \texttt{d\_int}$ coordinates of to be $0$. $\theta$ is varied on the $(\texttt{d\_int} +4 )$-dimensional cube $[0,1]^{\texttt{d\_int} +4}$. We generate $y = f_0(\bx) + \epsilon$, where $\epsilon$ are $\text{Normal}(0,0.1)$. For our experiment, we vary $m,n \in \{20, 40 , \ldots, 200\}$ and  $\texttt{d\_int} \in \{10,20\}$. \textcolor{black}{We train a three-layer neural network with ReLU activations and hidden layer widths set to $d$, using the Adam optimizer \citep{kingma2015adam} with a learning rate of 0.001. Training is performed via the Federated Averaging (FedAvg) algorithm over five communication rounds. In each round, participating clients perform one local epoch of training on their private data before sending model updates to the server, which aggregates them to update the global model. We repeat the entire procedure across $20$ independent runs and report the logarithm of the test Mean Squared Error (MSE) for both participating and non-participating clients in Figures~\ref{fig:11} and \ref{fig:12}.} We also conduct a similar experiment with $f_0^{(2)}(\bx) = \frac{1}{d}\sum_{i=1}^n x_i^2 \one\{x_i \le 0.5\} - \frac{1}{d}\sum_{i=1}^n (x_i - 3/4) \one\{x_i > 0.5\}$, which is a member of $\sH^1(\Real^d,\Real)$ and report the outcomes in Figures~\ref{fig:21} and \ref{fig:22}. It is clear from Figure~\ref{fig:1} that the error rates for $\texttt{d\_int} = 10$ is lower than for the case $\texttt{d\_int} = 20$, further reinforcing the evidence that the generalization performance of federated learning models are dependent on their intrinsic dimension only and not on the dimension of the representative feature space. The codes pertaining to this section are available at \texttt{\url{https://github.com/saptarshic27/FL}}.
\begin{figure}[!b]
    \centering
    \includegraphics[width=0.5\linewidth]{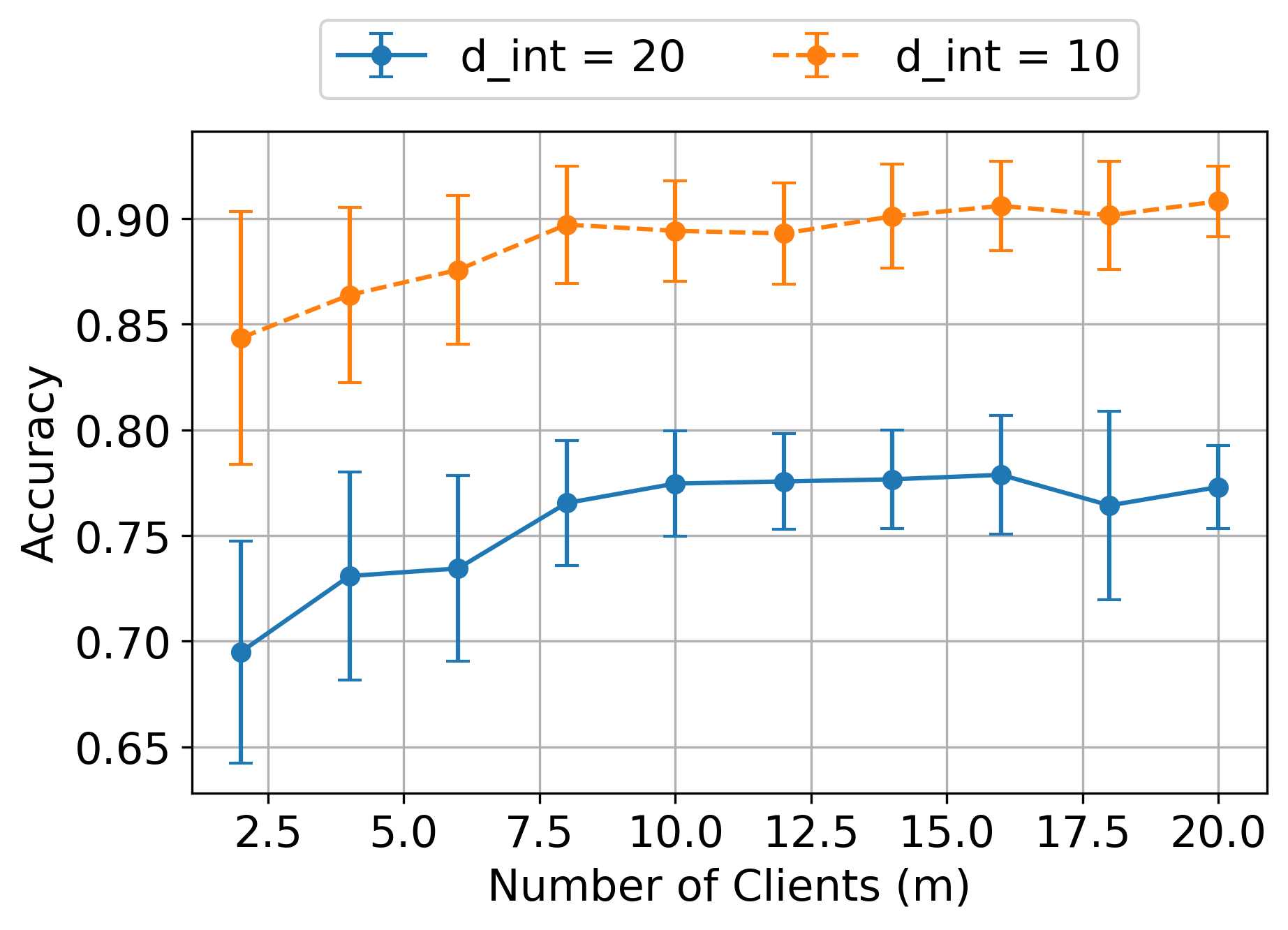}
    \caption{Average test accuracy for different values of the number of clients ($m$) for FedAvg on simulated data from Imagenet. The error bars denote the standard deviation out of $20$ replications. The test accuracy is heavily dependent on the intrinsic data dimension \texttt{d\_int} as a function of $m$ as predicted by Theorem~\ref{mainthm}.}
    \label{fig:imgnet}
\end{figure}
\subsection{Experiments on Imagenet}
Since it is difficult to assess the intrinsic dimensionality of natural images, we follow the prescription of \citet{pope2020intrinsic} and \citet{chakraborty2024a} to generate low-dimensional synthetic images. We use a pre-trained BigGAN \citep{brock2018large} with $128$ latent entries and outputs of size $128 \times 128 \times 3$, trained on the ImageNet dataset \citep{deng_imagenet}. We generate $6000$ images, from the classes, soap-bubble, volcano, goldfish, school-bus, space-shuttle, peacock, strawberry, castle, fire-truck, banana, where we fix most entries of the latent vectors to zero leaving only \texttt{d\_int} free
entries. We take \texttt{d\_int} to be $10$ and $20$, respectively.  We reduce the image sizes to $28 \times 28$ grayscale images for computational ease. In this study, we investigate the impact of client count on federated learning performance using a convolutional neural network (CNN) trained on $128 \times 128$ color images. We simulate a federated setting by partitioning a dataset into fixed-size subsets, assigning an equal number of training samples to each client. For a range of client counts ($m$), each with a fixed number of samples ($n = 3000$), we perform $20$ independent training runs using the Federated Averaging (FedAvg) algorithm over five communication rounds. Each client trains a local CNN for one epoch per round, and the resulting models are averaged to update the global model. The CNN consists of two convolutional layers followed by fully connected layers and is optimized using stochastic gradient descent with a learning rate of $0.01$. We evaluate the global model's accuracy on a held-out testset of size $10,000$ and report the average and standard deviation of test accuracies across runs. Our findings, visualized through Figure \ref{fig:imgnet}, reveal how model performance varies with increasing client count, for $\texttt{d\_int} \in \{10, 20\}$, showing that the model's test accuracy is heavily influenced by the intrinsic data dimension. 
\color{black}

\section{Main Results and Inference}\label{tana}

To facilitate the the theoretical analysis, we assume that the problem is smooth in terms of the learning function $f_0$. As a notion of smoothness, we assume that $f_0$ is $\beta$-H\"older. 
\textcolor{black}{This approach provides a natural framework for organizing models into a hierarchy based on their regularity and is standard in nonparametric statistical deep learning \citep{schmidt2020nonparametric,JMLR:v23:21-0732,chakraborty2024a,JMLR:v21:20-002, chen2019efficient, JMLR:v26:24-0054}. The smoothness parameter $\beta$ directly quantifies the model’s regularity, spanning models that range from highly irregular (small $\beta$) to very smooth (large $\beta$), allowing to study the generalization behavior across a wide spectrum of model complexity.}

\begin{assumption}\label{a1}
    $f_0 \in \sH^\beta([0,1]^d, \Real, C)$, for some positive constant $C>0$.
\end{assumption}
 For notational simplicity, we define, 
\[\tilde{d}_\alpha := \esssup_{\theta \in \Theta}^\pi\bar{d}_\alpha(\lambda_{\theta}), \]
i.e.,  the maximum entropic dimension of the explanatory variable for all clients. 
First, for precipitating clients, the error rate in terms of the total number of samples  depends on $\tilde{d}_{2\beta}$. In essence, $\|\hat{f} - f_0\|^2_{\fL_2(\lp)} $ scales roughly as $\tilde{\cO}((mn)^{-2\beta/(\tilde{d}_{2\beta}+2\beta)})$, barring poly-log factors, with high probability. This result is formally stated in Theorem~\ref{mainthm_2}.
\begin{restatable}[Error rate for participating clients]{theorem}{mainthmtwo}\label{mainthm_2}
     Suppose that $\lambda\left([0,1]^d\right) =1$ and $s>\tilde{d}_{2\beta}$. We can find an $n_0\in \mathbb{N}$, such that if $m,n \ge n_0$, we can choose $\cF = \cR\cN(L,W,B,R)$ in such a way that, $L \asymp \log\left(mn\right) $, $W \asymp \left(mn\right)^{\frac{s}{2\beta + s}} \log\left(mn\right)$, $\log B \asymp \log\left(mn\right)$ and $R \le 2C$, such that
    with probability at least $1-  2 \exp\left(-(mn)^{\frac{s}{2 \beta + s}}\right)$,
\begin{align*}
     \|\hat{f} - f_0\|^2_{\fL_2(\lp)}  \precsim & \, (mn)^{-\frac{2\beta}{s+2\beta}} \log^2 (mn).
 \end{align*}
\end{restatable}
Second, for nonparticipating clients, the error rate $\|\hat{f} - f_0\|^2_{\fL_2(\lambda)} $ exhibits a scaling behavior roughly characterized by 
\(\tilde{\cO}\left( \Delta(\theta, \bx) m^{-2\beta/(\bar{d}_{2\beta}(\lambda)+2\beta)} + (mn)^{-2\beta/(\bar{d}_{2\beta}(\lambda) + 2\beta)}\right),\) barring log-factors as shown in Theorem~\ref{mainthm}. Here the term,
\begin{align}
     \Delta(\theta,\bx) 
    =  \min\left\{\|\operatorname{KL}(\lambda_\theta, \lambda)\|_{\psi_1}, 1\right\} 
    = & \inf\left\{t>0: \E_\theta \exp(|\operatorname{KL}(\lambda_\theta, \lambda)|^p/t^p) \le 2\right\} \wedge 1. \label{eq_delta}
\end{align}
characterizes the level of dependency among $\theta$ and $X$. It essentially quantifies the ``closeness" of the distributions across clients, by evaluating how much the client-specific distribution $\lambda_\theta$ deviates from the mean distribution $\lambda$, given that $\theta \sim \pi(\cdot)$. When $\theta$ and $X$ are independent, it is straightforward to see that the discrepancy measure $\Delta(\theta; X) = 0$ as well. In this scenario, where the distributions of the explanatory variables across different clients are identical, the overall error rate behaves as though one has access to $mn$ i.i.d. samples, thus reflecting the optimal scenario for error scaling, i.e. $\Tilde{\cO}\left((mn)^{-2\beta/(\bar{d}_{2\beta}(\lambda) + 2\beta)}\right)$. However, when there is some degree of dependency between $\theta$ and $X$, the error rate no longer scales as favorably. Instead, it scales at a rate no faster than $\Tilde{\cO}\left( m^{-2\beta/(\bar{d}_{2\beta}(\lambda) + 2 \beta)} \max\left\{ \Delta(\theta,\bx) , n^{-2\beta/(\bar{d}_{2\beta}(\lambda) + 2\beta)}\right\}\right)$. Therefore, the extent to which the client distributions deviate from one another plays a significant role in determining the overall efficiency and performance for deep federated learners. When one has enough samples from each of the clients, i.e., when $n \ge \Delta(\theta,\bx)^{-\frac{\bar{d}_{2\beta}(\lambda) + 2\beta}{2 \beta }}$, the error rate scales as $\Tilde{\cO}\left( \Delta(\theta,\bx) m^{-2\beta/(\bar{d}_{2\beta}(\lambda) + 2 \beta)} \right)$, depending only on the number of participating clients and the discrepancy of the clients' distributions.

\begin{restatable}[Error rate for nonparticipating clients]{theorem}{mainthm}\label{mainthm}
     Suppose that $\lambda\left([0,1]^d\right) =1$ and $s>\bar{d}_{2\beta}(\lambda)$. We can find an $n_0^\prime \in \mathbb{N}$, such that if $m,n \ge n_0^\prime$,  we can choose $\cF = \cR\cN(L,W,B,R)$ in such a way that, $L \asymp \log\left(\frac{mn}{\Delta(\theta,\bx)n+1}\right) $, $W \asymp \left(\frac{mn}{\Delta(\theta,\bx)n+1}\right)^{\frac{s}{2\beta + s}} \log\left(\frac{mn}{\Delta(\theta,\bx)n+1}\right)$, $\log B \asymp \log\left(\frac{mn}{\Delta(\theta,\bx)n+1}\right)$ and $R \le 2C$, such that
    with probability at least $1-  3 \exp\left(-(mn)^{\frac{s}{2 \beta + s}}\right) - 2 \exp\left(-m^{\frac{s}{2 \beta + s}}\right)$,
\begin{align*}
     \|\hat{f} - f_0\|^2_{\fL_2(\lambda)} 
    \precsim  \, m^{-\frac{2\beta}{s+2\beta}} \left(\Delta(\theta,\bx) + n^{-\frac{2\beta}{s+2\beta}}\right)  \times \log^3 m \left(\log^2 m + \log n\right).
 \end{align*}
\end{restatable}

\paragraph{Comparison with Prior Art} 
Firstly, it is important to note that the current state of the art do not propose any dimension-based error bounds for deep federated learning. The only comparable results in the prior art is in the directions of Gaussian-noise additive regression models \citep{JMLR:v21:20-002} and GANs \citep{JMLR:v23:21-0732,dahal2022deep,JMLR:v26:24-0054}. The negative exponent for the sample size derived by \cite{JMLR:v21:20-002} is roughly, $\frac{2\beta}{2\beta + \overline{\text{dim}}_M(\lambda)}$. Since $\bar{d}_{2\beta} (\lambda) \le \overline{\text{dim}}_M(\lambda)$ \cite{JMLR:v26:24-0054}, the negative exponent of the sample size derived in this paper, i.e. $\frac{2\beta}{2\beta + \bar{d}_{2\beta}(\lambda)} \ge \frac{2\beta}{2\beta + \overline{\text{dim}}_M(\lambda)}$, resulting in better rates compared to the existing literature for additive regression models. 

Secondly, the prior art often do not address model misspecification. In the existing literature,  \citet{mohri2019agnostic} introduced the agnostic federated learning framework, deriving Rademacher complexity based bounds that account for client heterogeneity via the maximum $\chi^2$ divergence, though the resulting guarantees are conservative worst-case bounds. Subsequent work has refined this understanding. \citet{hu2023generalization} established fast-rates under a two-level distributional framework, extending guarantees to non-participating clients and unbounded losses, while  \citet{sun2024understanding} provide generalization bounds for gradient-based federated learners under different smoothness assumptions. Importantly, none of these analyses addresses model misspecification, which can lead to loose characterizations of the true generalization gap in heterogeneous settings. Further, we tighten the notion of heterogeneity by considering the Orlicz norm of the KL-divergence, rather than the maximum $\chi^2$ divergence considered by \citep{mohri2019agnostic} or the total variation distance considered by \citep{sun2024understanding}.

\paragraph{Implications of the Main Results}
Using the high probability bounds in Theorems~\ref{mainthm_2} and \ref{mainthm}, we can derive control the expected excess risk for both participating and nonparticipating clients. We state this result as a corollary as follows:
\begin{restatable}{corollary}{corone}
    \label{cor_exp}
Suppose that $\lambda\left([0,1]^d\right)=1$. Then,
\begin{enumerate}
\item[(a)] if $s>\tilde{d}_{2\beta}$ and if $\cF$ is chosen according to Theorem~\ref{mainthm_2}, then, \(\E \|\hat{f} - f^\ast\|^2_{\fL_2(\lp)}  \precsim   (mn)^{-\frac{2\beta}{s+2\beta}} \log^2(mn).\)
    \item[(b)] if $s>\bar{d}_{2\beta}(\lambda)$ and if $\cF$ is chosen according to Theorem~\ref{mainthm}, 
    \begin{align*}
       \E \|\hat{f} - f_0\|^2_{\fL_2(\lambda)} 
       \precsim & \, m^{-\frac{2\beta}{s+2\beta}} \left(\Delta(\theta,\bx) + n^{-\frac{2\beta}{s+2\beta}}\right) \times \log^3 m \left(\log^2 m + \log (mn)\right).
    \end{align*}
\end{enumerate}
\end{restatable}

Suppose that the explanatory variables are supported on a $d^\star$-dimensional compact differentiable manifold. From Propositions 8 and 9 of  \citet{weed2019sharp}, we note that the Minkowski and lower Wasserstein dimension of $\lambda$ is $d^\star$. Since $\bar{d}_\alpha(\lambda)$ lies between these two dimension \cite[Proposition 8]{JMLR:v26:24-0054}, we conclude that $\bar{d}_\alpha(\lambda) = d^\star$, for all $\alpha >0$. Hence the error rates for participating and nonparticipating clients scale as roughly $\cO\left((mn)^{-\frac{2\beta}{d^\star+2\beta}}\right)$ and $\cO\left(m^{-\frac{2\beta}{d^\star+2\beta}} + (mn)^{-\frac{2\beta}{d^\star+2\beta}}\right)$, respectively, excluding the excess log-factors. 
\begin{corollary}
    Suppose that the support of $\lambda$ is a compact $d^\star$-dimensional differentiable manifold and let $s>d^\star$. Suppose that the assumptions of Theorem~\ref{mainthm} hold and and $\cF$ is chosen according to Theorem~\ref{mainthm}, with probability at least $1-  3 e^{-(mn)^{\frac{s}{2 \beta + s}}} - 2 e^{-m^{\frac{s}{2 \beta + s}}}$,
\begin{align*}
     \|\hat{f} - f^\ast\|^2_{\fL_2(\lambda)} 
    \precsim & \, m^{-\frac{2\beta}{s+2\beta}} \left(\Delta(\theta,\bx) + n^{-\frac{2\beta}{s+2\beta}}\right)   \times \log^3 m \left(\log^2 m + \log (mn)\right).
 \end{align*}
\end{corollary}

One can also infer that the error rate for participating clients decays faster than the error rates for nonparticipating clients. This is because one can show that $\tilde{d}_{2\beta} \le \bar{d}_{2\beta}(\lambda)$, making the upper bound on the error converge faster for participating clients than that of nonparticipating clients. To show this, we state the Lemma~\ref{lem:6} which ensures that $\tilde{d}_\alpha$ is at most the $\alpha$-entropic dimension of $\lambda$. Thereafter, we state the result formally in Corollary~\ref{cor10}, which immediately follows from Lemma~\ref{lem:6} and Theorem~\ref{mainthm_2}.
\begin{restatable}{lemma}{lemsix}\label{lem:6}
    For any $\alpha>0$, $\tilde{d}_\alpha  \le \bar{d}_\alpha(\lambda)$.
\end{restatable}
\begin{corollary}\label{cor10}
   Suppose $\lambda\left([0,1]^d\right)$ and let $s>\bar{d}_{2\beta}(\lambda)$. Then if $\cF$ is chosen according to Theorem \ref{mainthm_2} and the assumptions of Theorem~\ref{mainthm_2} hold, with probability at least $1-  2 \exp\left(-(mn)^{\frac{s}{2 \beta + s}}\right)$,  \[ \|\hat{f} - f_0\|^2_{\fL_2(\lp)}  \precsim  \, (mn)^{-\frac{2\beta}{s+2\beta}} \log^2 (mn).\]
\end{corollary}
We observe that Theorem~\ref{mainthm_2} and \ref{mainthm} imply that one can select networks with the number of weights as an exponent of $m$ and $n$, which is smaller than 1. Additionally, this exponent is solely dependent on the intrinsic dimension of the explanatory variables. Furthermore, for smooth models, where $\beta$ is large, it is feasible to opt for smaller networks that necessitate fewer parameters compared to non-smooth models. This is because the exponent on the number of weights in Theorems~\ref{mainthm_2} and \ref{mainthm} decreases as $\beta$ increases. Such a trend aligns with practical expectations, where simpler problems often require less complex networks in contrast to more challenging problems.

\section{Proof of the Main Results} \label{pf}
This section provides a structured overview of proofs the main results, namely Theorems \ref{mainthm_2} and \ref{mainthm}, with comprehensive details available in the appendix. For ease of notation, we denote $\prob(\cdot|x,\theta)$ to represent the conditional distribution given $\{\bx_{i,j}\}_{i \in [m], j \in [n]}$ and $\{\theta_i\}_{i \in [m]}$. Similarly, $\prob(\cdot|\theta)$ is used to denote the conditional distribution given $\{\theta_i\}_{i \in [m]}$. 
As an initial step in establishing bounds on the excess risks for both the participating and nonparticipating clients, we proceed to derive the following oracle inequality. This inequality effectively constrains the excess risk in terms of the approximation error and a generalization gap.
\begin{restatable}{lemma}{lemone}\label{oracle}
     For any $f \in \cF$,  
     \begin{align}
          \|\hat{f} - f_0\|^2_{\fL_2(\hat{\lambda}_{m,n})} 
        \le & \|f - f_0\|^2_{\fL_2(\hat{\lambda}_{m,n})} + \frac{2}{mn}\sum_{i=1}^m \sum_{j=1}^n \epsilon_{ij} (  \hat{f}(\bx_{ij}) - f(x_{ij})).\label{e_11}
     \end{align}
\end{restatable}
\textcolor{black}{The first term on the right-hand side (RHS) of the oracle inequality (Lemma~\ref{oracle}) can be interpreted as an approximation error, reflecting how well functions in the chosen hypothesis class $\cF$ can represent the true regression function. The second term, by contrast, corresponds to a generalization gap, capturing the discrepancy between empirical performance on the sampled data and expected performance over the true distribution. These two components embody a fundamental trade-off in statistical learning: increasing the size or expressivity of the network class $\cF$ typically improves approximation accuracy, thereby reducing the first term. However, doing so also increases model complexity, which can exacerbate the generalization gap due to overfitting, reflected in the second term. Conversely, restricting the hypothesis class too severely may yield better generalization but at the cost of a large approximation error. Thus, one must balance these opposing sources of error to achieve optimal generalization. In particular, the network architecture should be selected to ensure that both terms are sufficiently controlled, thereby minimizing the total excess risk. In the results that follow, we analyze these components individually and derive explicit bounds that guide the choice of network size as a function of sample size, the number of clients, the intrinsic dimension of the data distribution, and the smoothness of the target function.}
\subsection{Generalization Gap}
To effectively manage the generalization error, we employ localization techniques, as expounded by \citet[Chapter 14]{wainwright_2019}. These techniques play a pivotal role in achieving fast convergence of the sample estimator to the population estimator under the $\fL_2(\lp)$ and $\fL_2(\lambda)$ norms. It is crucial to note that the true function $f_0$ may not be exactly representable by a ReLU network. In such cases, our alternative approach involves establishing a high-probability bound for the squared $\fL_2(\lambda)$ norm difference between our estimated function $\hat{f}$ and $f^\ast$, where $f^\ast \in \cF$ is considered sufficiently close to $f_0$. Our strategy unfolds in a two-step process: firstly, we derive a local complexity bound, detailed in Lemma~\ref{lem:3}. Subsequently, we leverage this local complexity bound to derive an estimate for $\|\hat{f} - f^\ast\|^2_{\fL_2(\lp)}$, as expounded in Lemma~\ref{lem:4}. This result is then utilized to control $\|\hat{f} - f^\ast\|^2_{\fL_2(\lambda)}$ in Lemma~\ref{lem:6.5}. These results are presented subsequently, with proofs available in the Appendix.

\begin{restatable}{lemma}{lemthree}
    \label{lem:3}
    Suppose $\alpha \in (0,1/2)$ and $n \ge \max\left\{e^{1/\alpha}, \operatorname{Pdim}(\cF)\right\}$. Then, for any $f^\ast \in \cF$, with probability (under $\prob(\cdot|\theta)$) at least, $1 -  \exp\left(-n^{1-2 \alpha}\right)$, 
    \begin{align}
        \|\hat{f} - f^\ast\|^2_{\fL_2(\hat{\lambda}_{m,n})} 
        \precsim   \|f^\ast - f_0\|^2_{\fL_2(\hat{\lambda}_{m,n})} + (mn)^{-2 \alpha} +  \frac{1}{mn }\operatorname{Pdim}(\cF) \log (mn).\label{ee_s5}
    \end{align}
\end{restatable}
\begin{restatable}{lemma}{lemfour}
    \label{lem:4}
     Suppose that $n \ge \operatorname{Pdim}(\cF)$. Then, with probability (under $\prob(\cdot|\theta)$) at least $1 - 3 \exp(-(mn)^{1-2\alpha})$,
\begin{align*}
    \|\hat{f}-f^\ast\|_{\fL_2(\lp)}^2 
    \precsim & \|f^\ast - f_0\|^2_{\fL_2(\lp)} + (mn)^{-2\alpha}  + \frac{1}{m n }\left(\operatorname{Pdim}(\cF) \log (m n) + \log \log (mn)\right)\\
    &   + \epsilon^2  +\frac{1}{mn}   \log \cN\left(\epsilon; \cF, \|\cdot\|_{\fL_\infty([0,1]^d)}\right)
\end{align*}
\end{restatable}
\begin{lemma}\label{lem:6.5}
  With probability at least $1-3 \exp\left(-(mn)^{1-\alpha}\right) - 2 \exp\left(-m^{1-\alpha^\prime}\right)$,
    \begin{align*}
     \|\hat{f}-f^\ast\|_{\fL_2(\lambda)}^2 
   & \precsim   \epsilon^2 + (mn)^{-2\alpha} + \frac{1}{mn}   \log \cN\left(\epsilon; \cF, \|\cdot\|_{\fL_\infty([0,1]^d)}\right)\nonumber\\
    & + \frac{1}{m n }\left(\operatorname{Pdim}(\cF) \log (m n) + \log \log (mn)\right)    \nonumber\\
    &
 + \frac{  \Delta(\theta,\bx)}{ m} \left(\log \cN\left(\epsilon; \cF, \|\cdot\|_{\fL_\infty([0,1]^d)}\right) + m^{1-2\alpha^\prime} + \log \log m\right).
\end{align*}
\end{lemma}

\subsection{Approximation Error}
To effectively bound the overall error in Lemma~\ref{oracle}, one needs to control the approximation error, denoted by the first term of \eqref{e_11}. Exploring the approximating potential of neural networks has witnessed substantial interest in the research community in the past decade or so. Pioneering studies such as those by \citet{cybenko1989approximation} and \citet{hornik1991approximation} have extensively examined the universal approximation properties of networks utilizing sigmoid-like activations. These foundational works demonstrated that wide, single-hidden-layer neural networks possess the capacity to approximate any continuous function within a bounded domain. In light of recent advancements in deep learning, there has been a notable surge in research dedicated to exploring the approximation capabilities of deep neural networks. Some important results in this direction include those by \citet{yarotsky2017error,lu2021deep,petersen2018optimal,shen2019nonlinear,schmidt2020nonparametric} among many others. To control the approximation errors $\|f^\ast - f_0\|^2_{\fL_2(\lp)}$ and $\|f^\ast - f_0\|^2_{\fL_2(\lambda)}$, we employ the recent approximation results derived by \citet{JMLR:v26:24-0054}. 
\begin{lemma}[\cite{JMLR:v26:24-0054}]\label{lem_approx}
    Suppose that $f \in \sH^\alpha(\Real^{d}, \Real,C)$, for some $C >0$ and let $s > \bar{d}_{\alpha p}(\mu)$. Then, we can find constants  $\epsilon_0$ and $a$, that might depend on $\alpha$, $d$ and $C$, such that, for any $\epsilon \in (0, \epsilon_0]$, there exists a ReLU network, $\hat{f}$ with $\cL(\hat{f}) \le a \log(1/\epsilon)$, $\cW(\hat{f}) \le a \log(1/\epsilon)  \epsilon^{-s/\alpha}$, $\cB(\hat{f}) \le a \epsilon^{-1/\alpha}$ and $\cR(\hat{f}) \le 2C$, that satisfies, 
    \(\|f-\hat{f}\|_{\fL_p(\gamma)} \le \epsilon .\) 
\end{lemma}
It is noteworthy that when $\text{supp}(\lambda)$ possesses a finite Minkowski dimension, as per \citet[Proposition~8 (c)]{JMLR:v26:24-0054}, we observe that $\bar{d}_{\alpha p} \le \overline{\text{dim}}_M(\mu)$. Consequently, the number of weights needed for an $\epsilon$-approximation, in the $\fL_p$ sense, is limited to at most $\cO(\epsilon^{-\bar{d}_{\alpha p}/\alpha} \log(1/\epsilon))$. This result improves upon the bounds $\cO(\epsilon^{-\overline{\text{dim}}_M(\mu)/\alpha})$,  derived by \citet{JMLR:v21:20-002} as a special case. It is crucial to highlight that the requisite number of weights for low-dimensional data, specifically when $\bar{d}_{\alpha p}(\gamma) \ll d$, is notably smaller than $\cO(\epsilon^{-d/\alpha} \log(1/\epsilon))$. This stands in contrast to the scenario when approximating over the entire space with respect to the $\ell_\infty$-norm \citep{yarotsky2017error, chen2019efficient}. Using the above results, we are now ready to formally prove Theorem~\ref{mainthm_2} and \ref{mainthm} in Sections \ref{pf1} and \ref{pf2}, respectively.  
\subsection{Proof of Theorem~\ref{mainthm_2}}
\label{pf1}
\begin{proof}
Suppose that $s>\tilde{d}_{2\beta}$. From Lemma~\ref{lem:6}, we observe that $\tilde{d}_{2\beta} \ge \bar{d}_{2\beta}(\lp)$, almost surely under $\prob(\cdot|\theta)$. Hence, $s>\bar{d}_{2\beta}(\lp)$ almost surely under $\prob(\cdot|\theta)$. From Lemma~\ref{lem_approx}, we can choose $\cF =\cR\cN(L,W,B,R)$ with $L \asymp \log(1/\epsilon)$, $W \asymp \epsilon^{-s/\beta}$, $\log B \asymp \log(1/\epsilon)$ and $R \le 2C$ such that $\inf_{f\in \cF} \|f - f_0\|_{\fL_2(\lp)} \le \epsilon$.
From Lemma~\ref{lem:4}, we observe that, under $\prob(\cdot|\theta)$, with probability at least, 
 $1-  3 \exp\left(-(mn)^{1-2\alpha}\right) $,
 \begingroup
 \allowdisplaybreaks
\begin{align}
    \|\hat{f} - f_0\|^2_{\fL_2(\lp)} 
    \le & 2 \|\hat{f} - f^\ast\|^2_{\fL_2(\lp)} + 2 \|f^\ast - f_0\|^2_{\fL_2(\lp)} \nonumber\\
    \precsim & \|f^\ast - f_0\|^2_{\fL_2(\lp)} + (mn)^{-2\alpha} +  \frac{1}{m n }\operatorname{Pdim}(\cF) \log (m n)  \nonumber\\
    & + \frac{\log \log (mn)}{mn} + \epsilon^2 + \frac{\log \cN\left(\epsilon; \cF, \|\cdot\|_{\fL_\infty([0,1]^d)}\right)}{mn} \nonumber\\
     \precsim & \epsilon^2  + (mn)^{-2\alpha} + \frac{\log (mn)}{mn} WL \log W    + \frac{\log \log (mn)}{mn} + \frac{W \log \left( \frac{ 2LB^L(W + 1)^L}{\epsilon} \right)}{mn}\label{e19}
\end{align}
\endgroup
We choose $\epsilon \asymp \left(mn\right)^{-\frac{\beta}{2 \beta + s}}$. Thus, from \eqref{e19}, we note that with probability at least, 
 $1-  3 \exp\left(-(mn)^{1-2\alpha}\right)$,
 \begin{align*}
     \|\hat{f} - f_0\|^2_{\fL_2(\lp)}  \precsim  (mn)^{-2\alpha} + \left(mn\right)^{-\frac{2\beta}{s+2\beta}} \log^3 (mn)
 \end{align*}
Choosing $\alpha = \frac{\beta}{2 \beta + s}$, we note that with probability at least $1-  3 \exp\left(-(mn)^{\frac{s}{2 \beta + s}}\right) $,
 \begin{align*}
     \|\hat{f} - f_0\|^2_{\fL_2(\lp)}  \precsim  \left(mn\right)^{-\frac{2\beta}{s+2\beta}} \log^3 (mn).
 \end{align*}
 Thus, for some  constant $\tau$,
 \begin{align*}
     & \prob\left(\|\hat{f} - f_0\|^2_{\fL_2(\lp)}  \le  \tau \left(mn\right)^{-\frac{2\beta}{s+2\beta}} \log^3 (mn) \bigg| \theta \right) 
      \ge   1-  3 \exp\left(-(mn)^{\frac{s}{2 \beta + s}}\right).
 \end{align*}
 The result now follows from Integrating both sides w.r.t. the distribution of $\theta_1, \dots, \theta_m$. 
 \end{proof}
 \subsection{Proof of Theorem~\ref{mainthm}}\label{pf2}
 \begin{proof}
From Lemma~\ref{lem_approx}, we can choose $\cF =\cR\cN(L,W,B,R)$ with $L \asymp \log(1/\epsilon)$, $W \asymp \epsilon^{-s/\beta}$, $\log B \asymp \log(1/\epsilon)$ and $R \le 2C$ such that $\inf_{f\in \cF} \|f - f_0\|_{\fL_2(\lambda)} \le \epsilon$.
From Lemma~\ref{lem:6.5}, we observe that, with probability at least, 
 $1-  3 \exp\left(-(mn)^{1-2\alpha}\right) - 2 \exp\left(-m^{1-2\alpha^\prime}\right)$,
 \begingroup
 \allowdisplaybreaks
\begin{align}
    &  \|\hat{f} - f^\ast\|^2_{\fL_2(\lambda)} \nonumber\\
    \precsim & \epsilon^2 + (mn)^{-2\alpha} + \frac{1}{mn}   \log \cN\left(\epsilon; \cF, \|\cdot\|_{\fL_\infty([0,1]^d)}\right) + \frac{1}{m n }\left(\operatorname{Pdim}(\cF) \log (m n) + \log \log (mn)\right)    \nonumber\\
    & + \frac{ \Delta(\theta,\bx)}{ m} \left(\log \cN\left(\epsilon; \cF, \|\cdot\|_{\fL_\infty([0,1]^d)}\right) + m^{1-2\alpha}\right) \nonumber\\
 \precsim & \epsilon^2 + (mn)^{-2\alpha} + \Delta(\theta,\bx) m^{-\alpha^\prime} + \frac{\log mn}{mn} \operatorname{Pdim}(\cF)  + \frac{\log \log mn}{mn} \nonumber\\
 & \quad + \left(\frac{1}{mn} + \frac{\Delta(\theta,\bx)}{m}\right) W \log \left( \frac{ 2LB^L(W + 1)^L}{\epsilon} \right) \label{e2018}
\end{align}
\endgroup
We choose $\epsilon \asymp \left(\frac{\Delta(\theta,\bx)}{m} + \frac{1}{mn}\right)^{\frac{\beta}{2 \beta + s}}$. Thus, from \eqref{e2018}, we note that with probability at least, 
 $1-  3 \exp\left(-(mn)^{1-2\alpha}\right) - 2 \exp\left(-m^{1-2\alpha^\prime}\right)$,
 \begin{align*}
       \|\hat{f} - f^\ast\|^2_{\fL_2(\lambda)}
       \precsim  \Delta(\theta,\bx) m^{-2\alpha^\prime} + (mn)^{-2\alpha}  + \left(\frac{\Delta(\theta,\bx)}{m}  + \frac{1}{mn}\right)^\frac{2\beta}{s+2\beta} \log^3 m \left(\log^2 m + \log (mn)\right)
 \end{align*}
Choosing $\alpha = \alpha^\prime = \frac{\beta}{2 \beta + s}$, we note that with probability at least $1-  3 \exp\left(-(mn)^{\frac{s}{2 \beta + s}}\right) - 2 \exp\left(-m^{\frac{s}{2 \beta + s}}\right)$, 
\begin{align*}
     \|\hat{f} - f^\ast\|^2_{\fL_2(\lambda)}  
    \precsim  m^{-\frac{2\beta}{s+2\beta}} \left(\Delta(\theta,\bx) + n^{-\frac{2\beta}{s+2\beta}}\right)\times\log^3 m \left(\log^2 m + \log (mn)\right).
 \end{align*}
 \end{proof}
 We refer the reader to the Appendix for a comprehensive and detailed exposition of the supporting lemmata and essential results.



\section{Discussions and Conclusion}\label{con}
In this paper, we present a comprehensive framework for analyzing error rates in deep federated regression, encompassing both participating and nonparticipating clients, particularly when the data manifests an intrinsically low dimensional structure within a high-dimensional feature space. We capture this intrinsic low-dimensionality using the entropic dimension of the explanatory variables and establish an error bound on  the excess risk, accounting for both misspecification and generalization errors. The derived excess risk bounds are achieved by balancing model misspecification against stochastic errors, enabling the identification of optimal network architectures based on sample size. This framework facilitates a nuanced analysis of model accuracy for both participating and nonparticipating clients, with a focus on the interplay between sample size and intrinsic data dimensionality.

Our contributions extend the existing literature by not only broadening parametric results to encompass more general nonparametric classes but also by incorporating a characterization of the ``closeness" of clients' distributions through the Orlicz-1 norm of the corresponding $\operatorname{KL}$-divergences in our generalization bounds--a consideration previously overlooked in prior studies. Supported by empirical evidence, we also provide a theoretical comparison of error rates for participating and nonparticipating clients, demonstrating that these rates depend not on the full-data dimensionality (in terms of the number of observations) but rather on the intrinsic dimension, thereby elucidating the effectiveness of federated learning in high-dimensional contexts.

While our findings shed light on the theoretical aspects of deep federated learning, it is crucial to recognize that practical evaluation of total test error must account for an optimization error component. Accurately estimating this component is a significant challenge due to the non-convex and complex nature of the optimization problem. Nevertheless, our error analyses remain independent of the optimization process and can be seamlessly integrated with optimization analyses. \textcolor{black}{The bounds derived in this paper are finite sample in nature. Specifically, our error bounds depend explicitly on both the number of participating clients $m$ and the cross-client heterogeneity term $\Delta(\theta, \bx)$, which quantifies the heterogeneity in data distributions across clients. As such, our framework remains informative even in low-participation regimes, although the guarantees naturally weaken as $m$ becomes small or $\Delta(\theta, \bx)$ becomes large. Exploring tighter bounds in these extreme regimes is an interesting direction for future work.} 
Examining the influence of the underlying divergence measure on the generalization bound could be an interesting direction for future research. One might attempt to incorporate the Wasserstein metric or Total Variation as a measure of dissimilarity between the clients' distributions, doing away with the density assumption w.r.t. a common dominating measure. Additionally, exploring the minimax optimality of the proposed bounds in terms of the sample sizes $m$ and $n$ presents an intriguing direction for future research.

\section*{Acknowledgements}
We gratefully acknowledge the support of the NSF and the Simons Foundation for the Collaboration on the
Theoretical Foundations of Deep Learning through awards DMS-2031883 and \#814639, the NSF’s support
of FODSI through grant DMS-2023505, and the support of the ONR through MURI award N000142112431.
\appendix
\tableofcontents
\section{Proofs from Section~\ref{tana}}

\subsection{Proof of Corollary~\ref{cor_exp}}
\begin{proof}
    We only prove part (b) of the corollary. Part (a) can be proved similarly. Suppose that a be the positive constant that honors the inequality in Theorem~8. 
    Then, with probability at least $1-  3 \exp\left(-(mn)^{\frac{s}{2 \beta + s}}\right) - 2 \exp\left(-m^{\frac{s}{2 \beta + s}}\right)$,
\begin{align*}
     \|\hat{f} - f^\ast\|^2_{\fL_2(\lambda)}  \le & \, a m^{-\frac{2\beta}{s+2\beta}} \left(1 + n^{-\frac{2\beta}{s+2\beta}}\right)  \log^3 m \left(\log^2 m + \log (mn)\right).
 \end{align*}
 We let $\xi  = a m^{-\frac{2\beta}{s+2\beta}} \left(1 + n^{-\frac{2\beta}{s+2\beta}}\right) 
      \log^3 m \left(\log^2 m + \log (mn)\right)$. Hence, 
      \begingroup
      \allowdisplaybreaks
 \begin{align*}
    &  \E  \|\hat{f} - f^\ast\|^2_{\fL_2(\lambda)} \\
     = & \E \|\hat{f} - f^\ast\|^2_{\fL_2(\lambda)} \one\{\|\hat{f} - f^\ast\|^2_{\fL_2(\lambda)} > \xi\}  + \E \|\hat{f} - f^\ast\|^2_{\fL_2(\lambda)} \one\{\|\hat{f} - f^\ast\|^2_{\fL_2(\lambda)} \le \xi\}\\
     \precsim & \prob(\|\hat{f} - f^\ast\|^2_{\fL_2(\lambda)} > \xi) + \xi\\
     \le & 3 \exp\left(-(mn)^{\frac{s}{2 \beta + s}}\right) + 2 \exp\left(-m^{\frac{s}{2 \beta + s}}\right)  + a m^{-\frac{2\beta}{s+2\beta}} \left(1 + n^{-\frac{2\beta}{s+2\beta}}\right)  \log^3 m \left(\log^2 m + \log (mn)\right)\\
     \precsim & m^{-\frac{2\beta}{s+2\beta}} \left(1 + n^{-\frac{2\beta}{s+2\beta}}\right)  \log^3 m \left(\log^2 m + \log (mn)\right), 
 \end{align*}
 \endgroup
 when $m$ and $n$ are large enough.
\end{proof}

\begin{lemma}\label{lem_bd_1}
   For any $\alpha>0$, $\bar{d}_\alpha(\lp) \le \tilde{d}_\alpha$, almost surely.
\end{lemma}
\begin{proof}
Let, $s> \max_{1 \le i \le m} \bar{d}_\alpha(\lambda_{\theta_i})$. By definition, we can find an $\epsilon_0 \in (0,1)$, such that if $\epsilon \in (0, \epsilon_0]$, $\sN_\epsilon(\lambda_{\theta_i}, \epsilon^\alpha) \le \epsilon^{-s}$, for all $i = 1, \ldots, m$. By definition, we can find sets $A_i$'s such that, $\lambda_{\theta_i}(A_i) \ge 1 - \epsilon^\alpha$ and $\cN(\epsilon; A_i, \varrho) \le \epsilon^{-s}$, for all $i = 1, \ldots, m$. Let $A = \cup_{i=1}^m A_i$. Then, $\lp(A) = \frac{1}{m} \sum_{i=1}^m \lambda_{\theta_i}(A) \ge \frac{1}{m} \sum_{i=1}^m \lambda_{\theta_i}(A_i) \ge 1 - \epsilon^\alpha$. Furthermore, $\cN(\epsilon;A, \varrho) \le \sum_{i=1}^m \cN(\epsilon; A_i, \varrho) \le m \epsilon^{-s}$. Thus, $\sN_\epsilon(\lp, \epsilon^\alpha) \le m \epsilon^{-s}$. Hence,
{\small
\[\bar{d}_\alpha(\lp) \le \limsup_{\epsilon \downarrow 0} \frac{\log \sN_\epsilon(\lp, \epsilon^\alpha)}{\log(1/\epsilon)} = \lim_{\epsilon \downarrow 0} \frac{\log m +s \log(1/\epsilon)}{\log (1/\epsilon)} = s.\]
}%
Thus, for any $s>\max_{1 \le i \le m} \bar{d}_\alpha(\lambda_{\theta_i})$, $\bar{d}_\alpha(\lp) \le s$. Hence, $\bar{d}_\alpha(\lp) \le \max_{1 \le i \le m} \bar{d}_\alpha(\lambda_{\theta_i}) \le \tilde{d}_\alpha$.
\end{proof}
 \subsection{Proof of Lemma~\ref{lem:6}}
 \begin{proof}
     Suppose that $s > \bar{d}_\alpha(\lambda) = \limsup_{\epsilon \downarrow 0} \frac{\log \sN_\epsilon(\lambda,\epsilon^\alpha)}{\log(1/\epsilon)}$. Thus, we can find $\epsilon_0 \in (0,1)$, such that if $\epsilon \in (0, \epsilon_0]$, $\sN_\epsilon(\lambda, \epsilon^\alpha) \le \epsilon^{-s}$. Hence there exists $S_\epsilon$, such that $\lambda(S_\epsilon) \ge 1 - \epsilon^\alpha$ and $\cN(\epsilon; S, \varrho) \le \epsilon^{-s}$. Suppose that \[\Theta_n = \left\{\theta \in \Theta:\lambda_\theta(S_{\epsilon/n}) \ge 1 - \epsilon^\alpha, \,\forall \epsilon \in (0,\epsilon_0]\right\}.\]
     For any $\epsilon \in (0,\epsilon_0]$, we note that,
     \begin{align*}
        \lambda(S_{\epsilon/n}) = &  \int \lambda_\theta(S_{\epsilon/n}) d\pi(\theta) \\
          = & \int_{\Theta_n} \lambda_\theta(S_{\epsilon/n}) d\pi(\theta) +  \int_{\Theta_n^\mathsf{c}} \lambda_\theta(S_{\epsilon/n}) d\pi(\theta) \\
         \le & \pi(\Theta_n) + \left(1- \epsilon^\alpha\right) \left(1-\pi(\Theta_n)\right)\\
         \implies 1-(\epsilon/n)^{\alpha} \le & 1- \epsilon_{k}^\alpha +  \epsilon_{k}^\alpha \pi(\Theta_2) \implies \pi(\Theta_2) \ge  1 - 1/n^\alpha.
    \end{align*}
    Further, note that if $\theta \in \Theta_n$, for all $\epsilon \in (0,\epsilon_0]$, $\lambda_\theta(S_{\epsilon/n}) \ge 1 - \epsilon^\alpha$ and
   \[\cN(\epsilon; S_{\epsilon/n}, \varrho) \le \cN(\epsilon/n; S_{\epsilon/n}, \varrho) \le n^s \epsilon^{-s}.\]
   Thus, $\sN_\epsilon(\lambda_\theta,\epsilon^\alpha) \le n^s \epsilon^{-s}$, for all $\epsilon \le \epsilon_0$. Thus,
   \[\bar{d}_\alpha(\lambda_\theta) = \limsup_{\epsilon \downarrow} \frac{\log \sN_\epsilon(\lambda\theta ,\epsilon^\alpha)}{\log(1/\epsilon)} \le s.\]
   Hence, $\Theta_n \subseteq \{\theta \in \Theta: \bar{d}_\alpha(\lambda_\theta) \le s\}$, for all $n \in \mathbb{N}$. Thus, $\pi\left(\{\theta \in \Theta: \bar{d}_\alpha(\lambda_\theta) > s\}\right) \le \pi(\Theta_n^\mathsf{c}) \le 1/n$, for all $n \in \mathbb{N}$. Taking $n \uparrow \infty$, we get that $\pi\left(\{\theta \in \Theta: \bar{d}_\alpha(\lambda_\theta) > s\}\right) = 0$. Thus, by definition, $s > \esssup_{\theta \in \Theta}^\pi \bar{d}_\alpha(\lambda_\theta)$, which proves the result.
 \end{proof}

\section{Proofs from Section~\ref{pf}}
\subsection{Proof of Lemma~\ref{oracle}}
\begin{proof}
    Since $\hat{f}$ is the global minimizer of $\sum_{i=1}^n (y_i - f(\bx_i))^2$, we note that, for any $f \in \cF$,
\begin{align}
    & \sum_{i=1}^m \sum_{j=1}^n (y_{ij} - \hat{f}(\bx_{ij}))^2 \le  \sum_{i=1}^m \sum_{j=1}^n (y_{ij} -f(\bx_{ij}))^2 \label{19.1}\\
    \iff &\sum_{i=1}^m \sum_{j=1}^n (f_0(x_{ij}) + \epsilon_{ij} - \hat{f}(\bx_{ij}))^2 \le  \sum_{i=1}^m \sum_{j=1}^n (f_0(x_{ij}) + \epsilon_{ij} -f(\bx_{ij}))^2.
\end{align}
Taking $f = f_0$, we get,
\begin{align}
     \sum_{i=1}^m \sum_{j=1}^n (f_0(x_{ij}) - \hat{f}(\bx_{ij}))^2 
    \le & \sum_{i=1}^m \sum_{j=1}^n (f_0(x_{ij}) - f(\bx_{ij}))^2 + 2 \sum_{i=1}^m \sum_{j=1}^n \epsilon_{ij} (  \hat{f}(\bx_{ij}) - f(x_{ij})) \nonumber\\
    \iff  \|\hat{f} - f_0\|^2_{\fL_2(\hat{\lambda}_{m,n})} 
    \le & \|f - f_0\|^2_{\fL_2(\hat{\lambda}_{m,n})} + \frac{2}{mn}\sum_{i=1}^m \sum_{j=1}^n \epsilon_{ij} (  \hat{f}(\bx_{ij}) - f(x_{ij})) \label{e2}
\end{align}

\end{proof}
\subsection{Proof of Lemma~\ref{lem:3}}
\begin{proof}
    We take $\delta = \max\left\{n^{-\alpha}, 2\|\hat{f} - f_0\|_{\fL_2(\hat{\lambda}_{m,n})}\right\} $ and let $\eta = e^{-(mn)^{1-2\alpha}}$. We consider two cases as follows.
    
\textbf{Case 1}: $\|\hat{f} - f^\ast\|_{\fL_2(\hat{\lambda}_{m,n})} \le \delta$

Then, by Lemma~\ref{lem:2}, with probability at least $1 -  \exp\left(-n^{1-2\alpha} \right)$
\begingroup
\allowdisplaybreaks

\begin{align}
    \|\hat{f} - f^\ast\|^2_{\fL_2(\hat{\lambda}_{m,n})} 
    \le & 2 \|\hat{f} - f_0\|^2_{\fL_2(\hat{\lambda}_{m,n})} + 2 \|f_0 - f^\ast\|^2_{\fL_2(\hat{\lambda}_{m,n})} \nonumber\\
    \precsim & \|f_0 - f^\ast\|^2_{\fL_2(\hat{\lambda}_{m,n})} + \sup_{g \in \sG_\delta} \frac{1}{n} \sum_{i=1}^n \epsilon_{ij} g(x_{ij}) \label{e14}\\
    \precsim & \|f_0 - f^\ast\|^2_{\fL_2(\hat{\lambda}_{m,n})} + \delta \sqrt{\frac{ \log (1/\eta)}{mn}} + \delta \sqrt{\frac{\operatorname{Pdim}(\cF) \log (m n/\delta)}{m n}} 
    \label{a_23.1}
\end{align}
\endgroup
In the above calculations, \eqref{e14} follows from \eqref{e2}. Inequality \eqref{a_23.1} follows from Lemma~\ref{lem:2}. Let $\alpha_1 \ge 1$ be the corresponding constant that honors the inequality in \eqref{a_23.1}. Then using the upper bound on $\delta$, we observe that,
\begingroup
\allowdisplaybreaks

\begin{align}
    & \|\hat{f} - f^\ast\|^2_{\fL_2(\hat{\lambda}_{m,n})} \nonumber\\
  \le & \alpha_1 \|f_0 - f^\ast\|^2_{\fL_2(\hat{\lambda}_{m,n})}  + \alpha_1 \delta \sqrt{\frac{\operatorname{Pdim}(\cF) \log (m n/\delta)}{m n}} + \alpha_1 (mn)^{-2\alpha }\nonumber\\
  \le & \alpha_1 \|f_0 - f^\ast\|^2_{\fL_2(\hat{\lambda}_{m,n})} + \frac{\delta^2 }{16} +  \frac{4 \alpha_1^2 }{mn}\operatorname{Pdim}(\cF) \log (m n/\delta)  + \alpha_1 (mn)^{-2\alpha } \label{e15}\\
    \le & \alpha_1 \|f_0 - f^\ast\|^2_{\fL_2(\hat{\lambda}_{m,n})} + (1/8 + \alpha_1) 
 (mn)^{-2 \alpha} + \frac{1}{4} \|\hat{f}- f_0\|^2_{\fL_2(\hat{\lambda}_{m,n})}  \nonumber\\
 & +  \frac{4(1 + \alpha)\alpha_1^2}{mn }\operatorname{Pdim}(\cF) \log (m n) \nonumber \\
    \le & \alpha_1 \|f_0 - f^\ast\|^2_{\fL_2(\hat{\lambda}_{m,n})} + 2 \alpha_1   (mn)^{-2\alpha} + \frac{1}{2} \|\hat{f}- f^\ast\|^2_{\fL_2(\hat{\lambda}_{m,n})}  + \frac{1}{2} \|f^\ast - f_0\|^2_{\fL_2(\hat{\lambda}_{m,n})} \nonumber\\
    & +  \frac{4(1 + \alpha)\alpha_1^2 }{m n }\operatorname{Pdim}(\cF) \log (m n) \nonumber
\end{align}

\endgroup
Here, \eqref{e15} follows from the fact that $\sqrt{xy} \le \frac{x}{16\alpha_1 } + 4 \alpha_1 y$, from the AM-GM inequality and taking $x = \delta^2$ and $y = \frac{\operatorname{Pdim}(\cF) \log (mn/\delta)}{m n}$. Thus,
\begin{align*}
    \|\hat{f} - f^\ast\|^2_{\fL_2(\hat{\lambda}_{m,n})} 
   \precsim & (m n)^{-2\alpha} +\|f^\ast - f_0\|^2_{\fL_2(\hat{\lambda}_{m,n})} +  \frac{1}{m n }\operatorname{Pdim}(\cF) \log (m n) .
\end{align*}
\textbf{Case 2}: $\|\hat{f} - f^\ast\|_{\fL_2(\hat{\lambda}_{m,n})} \ge \delta$

It this case, we note that $\|\hat{f} -f^\ast\|_{\fL_2(\hat{\lambda}_{m,n})} \ge 2 \|\hat{f} - f_0\|_{\fL_2(\hat{\lambda}_{m,n})}$. Thus,
\begin{align*}
    \|\hat{f} - f^\ast\|_{\fL_2(\hat{\lambda}_{m,n})}^2
   \le & 2 \|\hat{f} - f_0\|_{\fL_2(\hat{\lambda}_{m,n})}^2 + 2 \|f_0 - f^\ast\|_{\fL_2(\hat{\lambda}_{m,n})}^2 \\
    \le & \frac{1}{2} \|\hat{f} - f^\ast\|_{\fL_2(\hat{\lambda}_{m,n})}^2 + 2 \|f_0 - f^\ast\|_{\fL_2(\hat{\lambda}_{m,n})}^2 \\
    \implies  \|\hat{f} - f^\ast\|_{\fL_2(\hat{\lambda}_{m,n})}^2 \precsim & \|f_0 - f^\ast\|_{\fL_2(\hat{\lambda}_{m,n})}^2
\end{align*}
Thus, from the above two cases, with probability at least, $1 -  \exp\left(-(mn)^{1-2\alpha}\right)$, 
\begin{align}
    \|\hat{f} - f^\ast\|^2_{\fL_2(\hat{\lambda}_{m,n})} 
   \precsim & (mn)^{-2 \alpha} +\|f^\ast - f_0\|^2_{\fL_2(\hat{\lambda}_{m,n})} +  \frac{1}{n }\operatorname{Pdim}(\cF) \log (m n).\label{e_s3}
\end{align}

From equation \eqref{e_s3}, we note that, for some constant $B_4$, 
\begin{align*}
    & \prob\bigg( \|\hat{f} - f^\ast\|^2_{\fL_2(\hat{\lambda}_{m,n})} \le  B_4\big(  (mn)^{-2 \alpha} +\|f^\ast - f_0\|^2_{\fL_2(\hat{\lambda}_{m,n})}   + \frac{1}{m n }\operatorname{Pdim}(\cF) \log (m n) \big)\bigg| x, \theta \bigg) \\
    \ge & 1 -  \exp\left(-(mn)^{1-2\alpha}\right).
\end{align*}
Integrating both sides w.r.t.the joint distribution of $\{\bx_{ij}\}_{i \in [m], j \in [n]}$, we observe that under $\prob(\cdot|\theta)$, with probability at least, $1 -  \exp\left(-(mn)^{1-2 \alpha}\right)$, 
\begin{align}
    \|\hat{f} - f^\ast\|^2_{\fL_2(\hat{\lambda}_{m,n})} 
  \precsim & (mn)^{-2 \alpha} +\|f^\ast - f_0\|^2_{\fL_2(\hat{\lambda}_{m,n})} +  \frac{1}{m n }\operatorname{Pdim}(\cF) \log (mn).\label{e_s5}
\end{align}
\end{proof}
\subsection{Proof of Lemma~\ref{lem:4}}
\begin{proof}
    In the proof all probabilities and expectations are w.r.t. $\prob(\cdot|\theta)$. Suppose that $\cC\left(\epsilon;\cR\cN(L,W,B,R),\|\cdot\|_{\fL_\infty(\lambda)}\right) $ is an $\epsilon$-cover of $\cR\cN(L,W,B,R)$ w.r.t. the $\|\cdot\|_{\fL_\infty(\lambda)}$-norm and let, $N = \cN\left(\epsilon;\cR\cN(L,W,B,R),\|\cdot\|_{\fL_\infty(\lambda)}\right)$. Let, $f \in \cC\left(\epsilon;\cR\cN(L,W,B,R),\|\cdot\|_{\fL_\infty(\lambda)}\right)$ be such that, $\|f-\hat{f}\|_{\fL_\infty(\lambda)} \le \epsilon$. Then 
 \begin{align}
     \|\hat{f}-f^\ast\|_{\fL_2(\lp)}^2 
     \le  2 \|\hat{f}-f\|_{\fL_2(\lp)}^2 + 2 \|f-f^\ast\|_{\fL_2(\lp)}^2 
     \le  2 \epsilon^2 + 2 \|f-f^\ast\|_{\fL_2(\lambda)}^2. \label{e3}
 \end{align}
 For any $g \in \cC\left(\epsilon;\cR\cN(L,W,B,R),\|\cdot\|_{\fL_\infty(\lp)}\right)$, we let $Z_{ij} = (g(\bx_{ij}) - f^\ast(\bx_{ij}))^2 - \E (g(\bx_{ij}) - f^\ast(\bx_{ij}))^2$. Let, $u = \max\left\{v, \frac{1}{2}\|g - f^\ast\|_{\fL_2(\lp)}^2\right\}$.
 Clearly, 
 \begingroup
 \allowdisplaybreaks
 \begin{align*}
     \E Z_{ij}^2 =  \operatorname{Var} \left((g(\bx_{ij}) - f^\ast(\bx_{ij}))^2 \right)
     \le   \E (g(\bx_{ij}) - f^\ast(\bx_{ij}))^4 
     \le  4 R^2 \E (g(\bx_{ij}) - f^\ast(\bx_{ij}))^2
     \le  8 R^2 u.
 \end{align*}
 \endgroup
 Furthermore, $|Z_{ij}| \le 8R^2$. Thus, from Bernstein's inequality (Lemma~\ref{bernstein}),  we note that,
 
 \begin{align}
     &\prob\bigg(\sum_{i=1}^m \sum_{j=1}^n \E (g(\bx_{ij}) - f^\ast(\bx_{ij}))^2 \ge \sum_{i=1}^m \sum_{j=1}^n (g(\bx_{ij}) - f^\ast(\bx_{ij}))^2 + m n u\bigg)
      \le  \exp\left( -\frac{m n u}{24R^2}\right) \nonumber\\
       \implies & \prob\left(\|g - f^\ast\|_{\fL_2(\lp)}^2 \ge \|g - f^\ast\|_{\fL_2(\hat{\lambda}_{m,n})}^2 + u\right)  \le  \exp\left( -\frac{m n u}{24R^2}\right) \le \exp\left( -\frac{m n v}{24R^2}\right).
 \end{align}
 
 Thus, by union bound, 
 
 \begin{align*}
    & \prob\bigg(\|g - f^\ast\|_{\fL_2(\lp)}^2 \le \|g - f^\ast\|_{\fL_2(\hat{\lambda}_{m,n})}^2 + u,\, \forall \, g \in \cC\left(\epsilon;\cR\cN(L,W,B,R),\|\cdot\|_{\fL_\infty(\lambda)}\right)\bigg) \\
    \ge & 1-N \exp\left( -\frac{m n v}{24R^2}\right).
 \end{align*}
 
Thus, under $\lp$, with probability at least, $1-N \exp\left( -\frac{m n v}{24R^2}\right)$, the followings hold for all \\$g \in \cC\left(\epsilon;\cR\cN(L,W,B,R),\|\cdot\|_{\fL_\infty(\lambda)}\right)$,
\begin{align*}
     \|g - f^\ast\|_{\fL_2(\lp)}^2
    \le& \|g - f^\ast\|_{\fL_2(\hat{\lambda}_{m,n})}^2 + u\\
    \le & \|g - f^\ast\|_{\fL_2(\hat{\lambda}_{m,n})}^2 + v + \frac{1}{2} \|g - f^\ast\|_{\fL_2(\lp)}^2\\
    \implies  \|g - f^\ast\|_{\fL_2(\lp)}^2 \le &  2\|g - f^\ast\|_{\fL_2(\hat{\lambda}_{m,n})}^2 + 2v.
\end{align*}
Taking $v = \frac{24R^2}{mn} \left(\log N + (mn)^{1-2\alpha}\right)$, we note that, under $\lp$ with probability at least, $1-\exp(-(mn)^{1-2\alpha})$, for all $g \in \cC\left(\epsilon;\cR\cN(L,W,B,R),\|\cdot\|_{\fL_\infty(\lambda)}\right)$,
\[\|g - f^\ast\|_{\fL_2(\lp)}^2 \le  2\|g - f^\ast\|_{\fL_2(\hat{\lambda}_{m,n})}^2 + \frac{48R^2}{mn} \left(\log N + (mn)^{1-2\alpha}\right)\]
From, \eqref{e3}, we thus observe that  under $\lp$, with probability at least, $1-\exp(-(mn)^{1-2\alpha})$, 
\begin{align}
    \|\hat{f}-f^\ast\|_{\fL_2(\lp)}^2 
    \le &  2 \epsilon^2 + 4 \|f - f^\ast\|_{\fL_2(\hat{\lambda}_{m,n})}^2 + \frac{96R^2}{mn} \left(\log N + (mn)^{1-2\alpha}\right)\\
    \le & 4 \epsilon^2 + 8 \|\hat{f} - f^\ast\|_{\fL_2(\hat{\lambda}_{m,n})}^2 + \frac{96R^2}{mn} \left(\log N + (mn)^{1-2\alpha}\right) 
\end{align}
Applying Lemma~14, 
we note that under $\lp$, with probability at least, $1-2 \exp(-(mn)^{1-2\alpha})$, 

\begin{align}
    \|\hat{f}-f^\ast\|_{\fL_2(\lp)}^2 
    \precsim & 4 \epsilon^2 + (mn)^{-2\alpha} +\|f^\ast - f_0\|^2_{\fL_2(\hat{\lambda}_{n,m})}+  \frac{1}{m n }\operatorname{Pdim}(\cF) \log (m n) + \frac{\log \log (mn)}{mn} \nonumber\\
    &+ \frac{96R^2}{mn} \left(\log N + (mn)^{1-2\alpha}\right) \label{e4}
\end{align}
Applying Lemma~\ref{app_bern} with $t = (mn)^{-\alpha}$ and $g: \bx  \mapsto (f^\ast(\bx) - f_0(\bx))^2$, we note that, with probability at least, 
 $1-   \exp\left(-(mn)^{1-2\alpha}\right) $,
 
\begin{equation}\label{e5}
    \|f^\ast - f_0\|^2_{\fL_2(\hat{\lambda}_{n,m})} \le \|f^\ast - f_0\|^2_{\fL_2(\lp)} + (mn)^{-2\alpha}
\end{equation}
Combining \eqref{e4} and \eqref{e5}, we note that with probability at least $1 - 3 e^{-(mn)^{1-2\alpha}}$,
\begin{align}
   \|\hat{f}-f^\ast\|_{\fL_2(\lp)}^2 
    \precsim & 4 \epsilon^2 + (mn)^{-2\alpha} +\|f^\ast - f_0\|^2_{\fL_2(\lp)}+  \frac{1}{m n }\operatorname{Pdim}(\cF) \log (m n) + \frac{\log \log (mn)}{mn} \nonumber\\
    & + \frac{96R^2}{mn} \left(\log N + (mn)^{1-2\alpha}\right) 
\end{align}
\end{proof}
\subsection{Proof of Lemma~\ref{lem:6.5}}
To prove Lemma~16, we first state and prove Lemmata \ref{lem:5} and \ref{lem:7}.
\begin{lemma}
\label{lem:5}
    With probability at least, 
 $1-  3 \exp\left(-(mn)^{1-2\alpha}\right) - 2 \exp\left(-m^{1-2\alpha^\prime}\right)$,
 \begingroup
 \allowdisplaybreaks
\begin{align*}
   \|\hat{f} - f^\ast\|^2_{\fL_2(\lambda)} 
    \precsim & \|f^\ast - f_0\|^2_{\fL_2(\lambda)} + m^{-2\alpha^\prime} + (mn)^{-2\alpha}  + \operatorname{Pdim}(\cF) \left(\frac{\log^2 m}{m} + \frac{\log (mn)}{mn}\right)\\
     &     + \frac{\log \log m }{m }
     + \frac{\log \log (mn)}{mn}  + \epsilon^2 + \frac{\log \cN\left(\epsilon; \cF, \|\cdot\|_{\fL_\infty([0,1]^d)}\right)}{mn}.
\end{align*}
\endgroup
\end{lemma}

\begin{proof}
Suppose that $\cH = \left\{h: \theta \mapsto \int (f-f^\prime)^2 d\lambda_\theta : f \in \cF \right\}$. We note that if $n \ge \operatorname{Pdim}(\cF)$ and $r \in (0, r_0]$, then, the empirical Rademacher complexity can be bounded as, 
\begin{align}
    \sR_{m}\left(\cH; \{\theta_i\}_{i \in [m]}\right) & = \frac{1}{m}\E_{\bsigma}\sum_{i=1}^m \sigma_i h(\theta_i) \nonumber\\
    & \precsim  \sqrt{\frac{r \log(1/r) \operatorname{Pdim}(\cF) \log m}{m}} \label{e11}\\
     & \le \sqrt{\frac{(\operatorname{Pdim}(\cF))^2 \log m}{m^2} + r \frac{\operatorname{Pdim}(\cF) \log(m/e\operatorname{Pdim}(\cF)) \log m}{m}} \label{e13}
\end{align}
 Here, $\sigma_i's$ are independent Rademacher random variables. In the above calculations, \eqref{e11} follows from Lemma~\ref{lem_bd_rad_2} and \eqref{e13} follows from Lemma~\ref{lem:8}. From Lemma~\ref{lem_yousefi} we note that, the RHS of \eqref{e13} has a fixed point of $r^\ast$ and $r^\ast \precsim \frac{\operatorname{Pdim}(\cF) \log^2 m}{m}$. 
Then, by Theorem 6.1 of \citet{1444}, we note that with probability at least $1 - e^{-x}$,
for all $h \in \cH$,
\begin{equation}
    \int h d\pi \precsim B_3\left(\int h d\hat{\pi}_{m} + \frac{\operatorname{Pdim}(\cF) \log^2 m}{m } + \frac{x}{m } + \frac{\log \log m }{m }\right),  \label{e_s4}
\end{equation}

for some absolute constant $B_3$. Now, taking $x = (mn)^{1 - 2\alpha}$ in \eqref{e_s4}, we note that, with probability at least $1 - \exp\left(-m^{1-2\alpha^\prime}\right)$,
\begin{align}
   \|\hat{f} - f^\ast\|^2_{\fL_2(\lambda)} \precsim & m^{-2\alpha^\prime} +\|\hat{f} - f^\ast\|^2_{\fL_2(\lp)}+  \frac{1}{m  }\operatorname{Pdim}(\cF) \log^2 m + \frac{\log \log m }{m }\label{e_s7}
\end{align}
Combining \eqref{e_s7} with Lemma 15, 
 with probability at least $1-  3 \exp\left(-(mn)^{1-2\alpha}\right) - \exp\left(-m^{1-2\alpha^\prime}\right)$,
\begin{align}
     & \|\hat{f} - f^\ast\|^2_{\fL_2(\lambda)} \nonumber\\
    \precsim & \|f^\ast - f_0\|^2_{\fL_2(\lp)} + m^{-2\alpha^\prime} + (mn)^{-2\alpha} + \frac{1}{m  }\operatorname{Pdim}(\cF) \log^2 m  \nonumber\\
    & +  \frac{1}{m n }\operatorname{Pdim}(\cF) \log (m n)  + \frac{\log \log m }{m }
     + \frac{\log \log (mn)}{mn}  + \epsilon^2 + \frac{\log \cN\left(\epsilon; \cF, \|\cdot\|_{\fL_\infty([0,1]^d)}\right)}{mn} \label{e16}
\end{align}
Applying Lemma~\ref{app_bern} with $t = m^{-\alpha^\prime}$, $Z_i = \theta_i$ and $g: \theta  \mapsto \|f^\ast - f_0\|_{\fL_2(\lambda_\theta)}^2$, we note that, with probability at least, 
 $1-  3 \exp\left(-(mn)^{1-2\alpha}\right) - 2 \exp\left(-m^{1-2\alpha^\prime}\right)$,
 \begingroup
 \allowdisplaybreaks
\begin{align}
    \|\hat{f} - f^\ast\|^2_{\fL_2(\lambda)} 
    \precsim & \|f^\ast - f_0\|^2_{\fL_2(\lambda)} + m^{-2\alpha^\prime} + (mn)^{-2\alpha} + \frac{1}{m  }\operatorname{Pdim}(\cF) \log^2 m  \nonumber\\ 
    & +  \frac{1}{m n }\operatorname{Pdim}(\cF) \log (m n)  + \frac{\log \log m }{m }
     + \frac{\log \log (mn)}{mn}  \nonumber
     \\
     & + \epsilon^2 + \frac{\log \cN\left(\epsilon; \cF, \|\cdot\|_{\fL_\infty([0,1]^d)}\right)}{mn} \label{e17}
\end{align}
\endgroup
\end{proof}
\begin{lemma}\label{lem:7}
   Suppose that  $\lambda_\theta \ll \lambda $, almost surely under $\pi$. Then, with probability at least $1-3 \exp\left(-(mn)^{1-\alpha}\right) - 2 \exp\left(-m^{1-\alpha^\prime}\right)$,
    \begin{align*}
    \|\hat{f}-f^\ast\|_{\fL_2(\lambda)}^2 
    \precsim &  \epsilon^2 + (mn)^{-2\alpha} + \frac{1}{mn}   \log \cN\left(\epsilon; \cF, \|\cdot\|_{\fL_\infty([0,1]^d)}\right) \nonumber\\
    & + \frac{1}{m n }\left(\operatorname{Pdim}(\cF) \log (m n) + \log \log (mn)\right)    \nonumber\\
    &
 + \frac{ \|\operatorname{KL}(\lambda_\theta, \lambda)\|_{\psi_1}}{ m} \left(\log \cN\left(\epsilon; \cF, \|\cdot\|_{\fL_\infty([0,1]^d)}\right) + m^{1-2\alpha^\prime}\right).
\end{align*}
\end{lemma}
\begin{proof}
     Let $\cC\left(\epsilon;\cR\cN(L,W,B,R),\|\cdot\|_{\fL_\infty(\lambda)}\right) $ be an $\epsilon$-cover of $\cR\cN(L,W,B,R)$ w.r.t. the $\|\cdot\|_{\fL_\infty(\lambda)}$-norm and let, \[N = \cN\left(\epsilon;\cR\cN(L,W,B,R),\|\cdot\|_{\fL_\infty(\lambda)}\right).\] 
     Let, $f \in \cC\left(\epsilon;\cR\cN(L,W,B,R),\|\cdot\|_{\fL_\infty(\lambda)}\right)$ be such that, $\|f-\hat{f}\|_{\fL_\infty(\lambda)} \le \epsilon$. Then 
 \begin{align}
     \|\hat{f}-f^\ast\|_{\fL_2(\lambda)}^2 
     \le & 2 \|\hat{f}-f\|_{\fL_2(\lambda)}^2 + 2 \|f-f^\ast\|_{\fL_2(\lambda)}^2 
     \le  2 \epsilon^2 + 2 \|f-f^\ast\|_{\fL_2(\lambda)}^2. \label{e_3}
 \end{align}
  Suppose that $\lambda_\theta \ll \lambda $, almost surely under $\pi$. Then, by Radon-Nykodym theorem, the density of $\lambda_\theta$ w.r.t. $\lambda$ exists and is denoted by $p_\theta = \frac{d\lambda_\theta}{d \lambda}$. By Lemma \ref{lem21}, with probability at least, $1-2 N \exp\left(- \frac{c_3m v}{ \|\operatorname{KL}(\lambda_{\theta}, \lambda)\|_{\psi_1}}\right)$, for all $g \in \cC\left(\epsilon;\cR\cN(L,W,B,R),\|\cdot\|_{\fL_\infty(\lambda)}\right)$,
\begin{align*}
     \|g - f^\ast\|_{\fL_2(\lambda)}^2 \le &  2\|g - f^\ast\|_{\fL_2(\lp)}^2 + 2v.
\end{align*}
Taking $v = \frac{\|\operatorname{KL}(\lambda_\theta, \lambda)\|_{\psi_1}}{c_3 m} \left(\log (2N) + m^{1-2\alpha^\prime}\right)$, we note that with probability at least, $1-\exp(-m^{1-2\alpha^\prime})$, for all $g \in \cC\left(\epsilon;\cR\cN(L,W,B,R),\|\cdot\|_{\fL_\infty(\lambda)}\right)$,
\[\|g - f^\ast\|_{\fL_2(\lambda)}^2 \le  2\|g - f^\ast\|_{\fL_2(\lp)}^2 + \frac{2 \|\operatorname{KL}(\lambda_\theta, \lambda)\|_{\psi_1}}{c_3 m} \left(\log (2N) + m^{1-2\alpha^\prime}\right)\]
From, \eqref{e_3}, we thus observe that, with probability at least $1-\exp(-m^{1-2\alpha^\prime})$, 
\begin{align}
     \|\hat{f}-f^\ast\|_{\fL_2(\lambda)}^2  \precsim &   \epsilon^2 +  \|f - f^\ast\|_{\fL_2(\lp)}^2 + \frac{ \|\operatorname{KL}(\lambda_\theta, \lambda)\|_{\psi_1}}{ m} \left(\log (2N) + m^{1-2\alpha^\prime}\right)
\end{align}
Applying Lemma~15, 
we note that, with probability at least $1-2 \exp(-m^{1-2\alpha^\prime})$, 
{\small
\begin{align}
    \|\hat{f}-f^\ast\|_{\fL_2(\lambda)}^2 
    \precsim &  \epsilon^2 + (mn)^{-2\alpha} + \frac{1}{mn}   \log  \cN\left(\epsilon; \cF, \|\cdot\|_{\fL_\infty([0,1]^d)}\right) \nonumber\\
     & + \frac{1}{m n }\left(\operatorname{Pdim}(\cF) \log (m n) + \log \log (mn)\right)    \nonumber\\
    &
 + \frac{ \|\operatorname{KL}(\lambda_\theta, \lambda)\|_{\psi_1}}{ m} \left(\log \cN\left(\epsilon; \cF, \|\cdot\|_{\fL_\infty([0,1]^d)}\right) + m^{1-2\alpha^\prime}\right) \label{e204}
\end{align}
}%
Similarly, using Lemma \ref{lem21}, we note that,
with probability at least, 
 $1-   \exp\left(-m^{1-2\alpha^\prime}\right) $,
 
\begin{equation}\label{e205}
    \|f^\ast - f_0\|^2_{\fL_2(\lp)} \precsim \|f^\ast - f_0\|^2_{\fL_2(\lambda)} + \sqrt{\|\operatorname{KL}(\lambda_{\theta}, \lambda)}\|_{\psi_1} m^{-2\alpha}
\end{equation}
Combining equations \eqref{e204} and \eqref{e205}, we note that with probability at least $1 - 3 \exp(-(mn)^{1-2\alpha}) - 2\exp(-m^{1-2\alpha^\prime})$,
\begin{align*}
    \|\hat{f}-f^\ast\|_{\fL_2(\lambda)}^2 
    \precsim &  \epsilon^2 + (mn)^{-2\alpha} + \frac{1}{mn}   \log \cN\left(\epsilon; \cF, \|\cdot\|_{\fL_\infty([0,1]^d)}\right) \nonumber\\
    & + \frac{1}{m n }\left(\operatorname{Pdim}(\cF) \log (m n) + \log \log (mn)\right)    \nonumber\\
    &
 + \frac{ \|\operatorname{KL}(\lambda_{\theta}, \lambda)\|_{\psi_1}}{ m} \left(\log \cN\left(\epsilon; \cF, \|\cdot\|_{\fL_\infty([0,1]^d)}\right) + m^{1-2\alpha^\prime}\right).
\end{align*}
\end{proof}
Combining Lemmata \ref{lem:5} and \ref{lem:7}, we get Lemma 16.
\section{Auxiliary Results}
\subsection{Supporting Results From the Literature}
This section outlines, without proof, a selection of relevant theoretical underpinnings from the literature that are employed in this paper. 
 \begin{lemma}[Bernstein’s Inequality for Bounded Distributions, Theorem~2.8.4 of \cite{vershynin2018high}]\label{bernstein}
Let $X_1, \ldots, X_N$ be independent, mean zero random variables such that $|X_i| \leq K$ for all $i \in [N]$. Then, for every $t \geq 0$, we have
\[
\prob \left(\left|\sum_{i=1}^N X_i \right|\geq t\right) \leq 2\exp\left(-\frac{t^2}{2\sigma^2 + \frac{Kt}{3}}\right).
\]
Here, $\sigma^2 = \sum_{i=1}^N \mathbb{E}(X_i^2)$ is the variance of the sum.
\end{lemma}

\begin{lemma}[Lemma B.1 of \cite{yousefi2018local}]\label{lem_yousefi}
    Let $c_1, c_2 > 0$ and $s > q > 0$. Then the equation
\(x^s - c_1x^q - c_2 = 0\)
has a unique positive solution $x_0$ satisfying
\[x_0 \leq \left(c_1^{\frac{s}{s-1}}  + \frac{sc_2}{s-q}\right)^{\frac{1}{s}}.\]
Moreover, for any $x \geq x_0$, we have $x^s \geq c_1x^q + c_2$.
\end{lemma}
\begin{lemma}[Lemma 21 of \cite{JMLR:v21:20-002}]\label{lem_nakada} 
     Let $\cF = \cR\cN(W, L, B)$ be a space of ReLU networks with the number of weights, the number of layers, and the maximum absolute value of weights bounded by $W$, $L$, and $B$ respectively. Then,
\[
\log\cN(\epsilon; \cF, \ell_\infty) \leq W \log \left( \frac{ 2LB^L(W + 1)^L}{\epsilon} \right).
\]
 \end{lemma}
 \begin{lemma}[Theorem 12.2 of \cite{anthony1999neural}]\label{lem_anthony_bartlett} 
     Assume for all $f \in \cF$, $\|f\|_{\infty} \leq M$. Denote the pseudo-dimension of $\cF$ as $\text{Pdim}(\cF)$, then for $n \geq \text{Pdim}(\cF)$, we have for any $\epsilon$ and any $X_1, \ldots, X_n$,
\[
\cN(\epsilon; \sF_{|_{X_{1:n}}}, \ell_\infty) \leq \left( \frac{2eM  n}{\epsilon\text{Pdim}(\cF)}\right)^{\text{Pdim}(\cF)}.
\]
 \end{lemma}
\subsection{Additional Lemmata}
\begin{lemma}\label{lem:1}
    Suppose that $Z_1, \dots, Z_n$ are independent and identically distributed sub-Gaussian random variables with variance proxy $\sigma^2$ and suppose that $\|f\|_\infty \le b$ for all $f \in \cF$. Then with probability at least $1-\delta$, 
    \[ \frac{1}{n}\sup_{f \in \cF} \sum_{i=1}^n Z_i f(x_i) - \frac{1}{n}\E\sup_{f \in \cF} \sum_{i=1}^n Z_i f(x_i) \precsim b \sigma \sqrt{\frac{\log(1/\delta)}{n}}.\]
\end{lemma}
\begin{proof}
    Let $g(Z) = \frac{1}{n}\sup_{f \in \cF} \sum_{i=1}^n Z_i f(x_i)$. Using the notations of \citet{maurer2021concentration}, we note that 
    \begingroup
    \allowdisplaybreaks
    \begin{align}
        \|g_k(Z)\|_{\psi_2} = & \frac{1}{n} \bigg\|\sup_{f \in \cF} \left(\sum_{i\neq k} z_i f(x_i) + Z_k f(x_k) \right)  - \E_{Z_k^\prime} \sup_{f \in \cF} \left(\sum_{i\neq k} z_i f(x_i) + Z_k^\prime f(x_k) \right) \bigg\|_{\psi_2} \nonumber\\
        \le & \frac{1}{n} \left\|\E_{Z_k^\prime}|Z_k-Z_k^\prime f(x_k)| \right\|_{\psi_2} \nonumber\\
        \le & \frac{b}{n} \left\|\E_{Z_k^\prime} |Z_k-Z_k^\prime| \right\|_{\psi_2} \label{e1}\\
        \le & \frac{b}{n} \left\|Z_k-Z_k^\prime \right\|_{\psi_2} \nonumber\\
        \le & \frac{2b}{n} \left\|Z_k\right\|_{\psi_2} \nonumber\\
        \precsim & \frac{b \sigma}{n} \nonumber.
    \end{align}
    \endgroup
    Here, \eqref{e1} follows from \citet[Lemma 6]{maurer2021concentration}. Thus, $\left\|\sum_{k=1}^n\|g_k(Z)\|_{\psi_2}^2 \right\|_\infty \precsim b^2\sigma^2/n$.  Hence applying \citet[Theorem 3]{maurer2021concentration}, we note that with probability at least $1-\delta$, 
     \[ \frac{1}{n}\sup_{f \in \cF} \sum_{i=1}^n Z_i f(x_i) - \frac{1}{n}\E\sup_{f \in \cF} \sum_{i=1}^n Z_i f(x_i) \precsim b \sigma \sqrt{\frac{\log(1/\delta)}{n}}.\]
\end{proof}
\begin{lemma}\label{lem:2}
    Suppose that 
\(\sG_\delta = \left\{f- f^\prime: \|f-f^\prime\|_{\fL_\infty(\hat{\lambda}_{m,n})} \le \delta \text{ and } f,f^\prime \in \cF \right\}\), with $\delta \le 1/e$. Also let, $n \ge \operatorname{Pdim}(\cF)$. Then, for any $t>0$, under $\prob(\cdot|\bx_{1:n})$, with probability  at least $1 -  \eta$,
\begin{align*}
    \sup_{g \in \sG_\delta} \frac{1}{mn}\sum_{i=1}^m \sum_{j=1}^n \epsilon_{ij} g(x_{ij}) 
    \precsim &  \delta \sqrt{\frac{ \log (1/\eta)}{mn}} + \delta \sqrt{\frac{\operatorname{Pdim}(\cF) \log (m n/\delta)}{m n}}
\end{align*}
\end{lemma}

\begin{proof}
From the definition of $\sG_\delta$, it is clear that \[\log \cN(\epsilon; \sG_\delta, \|\cdot\|_{\fL_\infty(\hat{\lambda}_{m,n})}) \le 2 \log \cN(\epsilon/2 ; \cF, \|\cdot\|_{\fL_\infty(\hat{\lambda}_{m,n})}). \] 
Let $Z_f = \frac{1}{\sqrt{mn}} \sum_{i=1}^m \sum_{j=1}^n \epsilon_{ik} f(x_{ij})$. Clearly, $\E_\epsilon Z_f = 0$. Furthermore,  we observe that,
\begingroup
\allowdisplaybreaks
\begin{align*}
     \E_\epsilon \exp(\lambda (Z_f-Z_g) )
    = &  \E_\epsilon \exp\left( \frac{\lambda}{\sqrt{mn}} \sum_{i=1}^m \sum_{j=1}^n \epsilon_{ij} (f(x_{ij}) -g (x_{ij}))\right)\\
    = & \prod_{i=1}^m \prod_{j=1}^n\E_\epsilon \exp\left( \frac{\lambda}{\sqrt{mn}} \epsilon_{ij} (f(x_{ij}) -g (x_{ij}))\right)\\
    \le & \prod_{i=1}^m \prod_{j=1}^n \E_\epsilon \exp\left( \frac{\lambda^2\sigma^2}{2n} (f(x_{ij}) -g (x_{ij}))^2\right)\\
    \le & \exp\left( \frac{\lambda^2 \sigma^2}{2mn} \sum_{i=1}^m \sum_{j=1}^m  (f(x_{ij}) -g (x_{ij}))^2 \right)\\
    = & \exp\left( \frac{\lambda^2 \sigma^2 }{2} \|f - g\|_{\fL_2(\hat{\lambda}_{m,n})}^2 \right).
\end{align*}
\endgroup
Thus, $(Z_f-Z_g)$ is $ \|f - g\|_{\fL_2(\hat{\lambda}_{m,n})}^2 \sigma^2$-subGaussian. Furthermore, 
\begingroup
\allowdisplaybreaks
\begin{align*}
     \sup_{f,g \in \sG_\delta} \|f - g\|_{\fL_2(\hat{\lambda}_{m,n})} = &
      \sup_{f,f^\prime \in \cF: \|f - f^\prime\|_{\fL_\infty(\hat{\lambda}_{m,n})} \le \delta} \| f - f^\prime\|_{\fL_2(\hat{\lambda}_{m,n})} \\
    \le & \sup_{f,f^\prime \in \cF: \|f - f^\prime\|_{\fL_\infty(\hat{\lambda}_{m,n})} \le \delta} \| f - f^\prime\|_{\fL_\infty(\hat{\lambda}_{m,n})} \\
    \le  & \delta.
\end{align*}
\endgroup

From \citet[Proposition 5.22]{wainwright_2019}, 
\begingroup
\allowdisplaybreaks
\begin{align}
 \E_\epsilon \sup_{g \in \sG_\delta} \frac{1}{\sqrt{mn}}\sum_{i=1}^m \sum_{j=1}^n \epsilon_{ij} g(x_{ij}) 
    = & \E_\epsilon \sup_{g \in \sG_\delta} Z_g \\
    = & \E_\epsilon \sup_{g \in \sG_\delta} (Z_g - Z_{g^\prime}) \nonumber \\
    \le & \E_\epsilon \sup_{g, g^\prime \in \sG_\delta} (Z_g - Z_{g^\prime}) \nonumber \\
    \le & 32 \int_0^{ \delta} \sqrt{ \log \cN(\epsilon; \sG_\delta, \fL_2(\hat{\lambda}_{m,n}) )} d\epsilon \nonumber \\
    \precsim &  \int_0^{\delta} \sqrt{\log \cN(\epsilon/2; \cF, \fL_\infty(\hat{\lambda}_{m,n}) )} d\epsilon \nonumber \\
    \precsim & \int_0^{\delta} \sqrt{\operatorname{Pdim}(\cF) \log (mn/\epsilon)  } d\epsilon \nonumber \\
    \le  &  \delta \sqrt{\operatorname{Pdim}(\cF) \log m n} + \sqrt{\operatorname{Pdim}(\cF)}\int_0^{ \delta}  \sqrt{\log (1/\epsilon)}   d\epsilon \nonumber \\
    \le & \delta \sqrt{\operatorname{Pdim}(\cF) \log mn} + 2 \sqrt{\operatorname{Pdim}(\cF)}  \delta \sqrt{ \log( 1/\delta)} \label{e12}\\
    \precsim & \delta \sqrt{\operatorname{Pdim}(\cF) \log(mn/\delta)}
    \end{align}
\endgroup
\eqref{e12} follows from Lemma~\ref{lem_17.1}. Thus,
\begin{align}
    \E_\epsilon \sup_{g \in \sG_\delta} \frac{1}{m n}\sum_{i=1}^m \sum_{j=1}^n \epsilon_{ij} g(x_{ij}) \precsim &  \delta \sqrt{\frac{\operatorname{Pdim}(\cF) \log (mn/\delta)}{mn}}
\end{align}
Applying Lemma~\ref{lem:1}, we note that, with probability at least $1 - \eta$,
\begin{align}
    \sup_{g \in \sG_\delta} \frac{1}{mn}\sum_{i=1}^m \sum_{j=1}^n \epsilon_{ij} g(x_{ij}) 
   \precsim &  \delta \sqrt{\frac{ \log (1/\eta)}{mn}} + \delta \sqrt{\frac{\operatorname{Pdim}(\cF) \log (m n/\delta)}{m n}}.\label{e10}
\end{align}
\end{proof}

\begin{lemma}\label{lem_bd_rad}
    Let $\cH_r = \{h = (f - f^\prime)^2 : f, f^\prime \in \cF \text{ and } \lambda_n h \le r\}$ with $\sup\limits_{f\in \cF} \|f\|_{\fL_\infty(\lambda_n)}< \infty$. Then, we can find $r_0>0$, such that if $0<r\le r_0$ and $n \ge \operatorname{Pdim}(\cF)$,  
    \[\E_{\epsilon} \sup_{h \in \cH_r} \frac{1}{n}\sum_{i=1}^n \epsilon_i h(x_i) \precsim \sqrt{\frac{r \log(1/r) \operatorname{Pdim}(\cF) \log n}{n}}.\]
\end{lemma}
\begin{proof}
  Let $B= 4 \sup_{f\in \cF} \|f\|_{\fL_\infty(\lambda_n)}^2$. We first fix $\epsilon \le \sqrt{2B r}$ and let $h = f - f^\prime$ be a member of $\cH_r$ with $f, f^\prime \in \cF$. We use the notation $\cF_{|_{x_{1:n}}} =\{(f(x_1), \dots, f(x_n))^\top: f \in \cF\}$. Suppose that $\bv^{f}, \bv^{f^\prime} \in \cC(\epsilon; \cF_{|_{x_{1:n}}}, \|\cdot\|_\infty)$ be such that $|\bv^f_i - f(x_i)|,|\bv^{f^\prime}_i - f^\prime(x_i)| \le \epsilon $, for all $i$. Here $\cC(\epsilon; \cF_{|_{x_{1:n}}}, \|\cdot\|_\infty)$ denotes the $\epsilon$ cover of $\cF_{|_{x_{1:n}}}$ w.r.t. the $\ell_\infty$-norm. 
  Let $\bv = \bv^f - \bv^{f^\prime}$ Then
  \begingroup
  \allowdisplaybreaks
  \begin{align}
       \frac{1}{n}\sum_{i=1}^n (h(x_i) - v_i^2 )^2 
      = & \frac{1}{n}\sum_{i=1}^n ((f(x_i) - f^\prime(x_i))^2 - (v_i^f - v_i^{f^\prime})^2 )^2 \nonumber\\
      \le &  \frac{2}{n}\sum_{i=1}^n ((f(x_i) - f^\prime(x_i))^2 + (v_i^f - v_i^{f^\prime})^2 )  \times ((f(x_i) - f^\prime(x_i)) - (v_i^f - v_i^{f^\prime}))^2 \label{e_s1}\\
      \precsim & \epsilon^2.
  \end{align}
  \endgroup
  Here \eqref{e_s1} follows from the fact that $(t^2 - r^2)^2 = (t+r)^2 (t-r)^2 \le 2 (t^2+r^2)(t-r)^2$, for any $t,r \in \Real$.  
Hence, from the above calculations, $\cN(\epsilon; \cH_r, \fL_2(\lambda_n)) \le \left(\cN\left( a_1 \epsilon; \cF, \fL_\infty(\lambda_n)\right)\right)^2$, for some absolute constant $a_1$. 
  \begingroup
  \allowdisplaybreaks
  \begin{align*}
      \operatorname{diam}^2(\cH_r, \fL_2(\lambda_n)) = \sup_{h, h^\prime \in \cH_r} \|h - h^\prime\|_{\fL_2(\lambda_n)}^2
      \le &  \sup_{h, h^\prime \in \cH_r} \frac{1}{n}\sum_{i=1}^n (h(x_i) - h^\prime(x_i))^2     \\
      \le & 2  \sup_{h \in \cH_r} \frac{1}{n}\sum_{i=1}^n h^2(x_i) \\
      \le & 2 B \sup_{h \in \cH_r} \frac{1}{n}\sum_{i=1}^n h(x_i)\\
      \le & 2 B r.
  \end{align*}
  \endgroup
  Hence, $\operatorname{diam}(\cH_r, \fL_2(\lambda_n)) \le  \sqrt{2 B r} $.
  Thus from \citet[Theorem 5.22]{wainwright_2019}
  \begingroup
  \allowdisplaybreaks
\begin{align}
     \E_{\epsilon} \sup_{h \in \cH_r} \frac{1}{n}\sum_{i=1}^n \epsilon_i h(x_i)
    \precsim & \int_0^{\sqrt{2 B r}} \sqrt{\frac{1}{n} \log \cN(\epsilon; \cH_r, \fL_2(\lambda_n))} d\epsilon \nonumber \\
    \le & \int_0^{\sqrt{2 B r}} \sqrt{\frac{2 \operatorname{Pdim}(\cF)}{n} \log \left(\frac{a_2 n}{\epsilon}\right)} d\epsilon \nonumber \\
    \precsim & \sqrt{2 B r} \sqrt{\frac{\operatorname{Pdim}(\cF) \log n}{n}} + \int_0^{\sqrt{2 B r}} \sqrt{\frac{\operatorname{Pdim}(\cF)}{n} \log (a_2/\epsilon)} d\epsilon \nonumber\\
    \precsim &  \sqrt{\frac{r \log(1/r) \operatorname{Pdim}(\cF) \log n}{n}}. \label{e6}\\
    \le & \sqrt{\frac{(\operatorname{Pdim}(\cF))^2 \log n}{n^2} + r \frac{\operatorname{Pdim}(\cF) \log(n/e\operatorname{Pdim}(\cF)) \log n}{n}}\label{e7}.
\end{align} 
\endgroup
Here, \eqref{e6} follows from Lemma~\ref{lem_17.1}. Here, \eqref{e7} follows from Lemma~\ref{lem:8} with $x = r$ and $y = \operatorname{Pdim}(\cF)/n$. 
\end{proof}

\begin{lemma}\label{lem_bd_rad_2}
    Let $\cH_r = \{h: \theta \mapsto \int (f - f^\prime)^2 d\lambda_\theta : f, f^\prime \in \cF \text{ and } \hat{\pi}_m h \le r\}$ with $\sup_{f\in \cF} \|f\|_{\fL_\infty(\lambda_n)}< \infty$. Then, we can find $r_0>0$, such that if $0<r\le r_0$ and $n \ge \operatorname{Pdim}(\cF)$,  
    \begin{align*}
         \E_{\sigma} \sup_{h \in \cH_r} \frac{1}{m}\sum_{i=1}^m \sigma_i h(\theta_i) 
        \precsim & \sqrt{\frac{(\operatorname{Pdim}(\cF))^2 \log m}{m^2} + r \frac{\operatorname{Pdim}(\cF) \log(m/e\operatorname{Pdim}(\cF)) \log m}{m}}.
    \end{align*}
\end{lemma}
\begin{proof}
  Let $B= 4 \sup_{f\in \cF} \|f\|_{\fL_\infty([0,1]^d)}$. Let $\hat{f}, \hat{f}^\prime \in \cC(\epsilon; \cF, \|\cdot\|_{\fL_\infty([0,1]^d)})$ be such that, $\|f - \hat{f}\|_{\fL_\infty([0,1]^d)}, \|f^\prime - \hat{f}^\prime\|_{\fL_\infty([0,1]^d)} \le \epsilon$. Let $\hat{h}(\theta) = \int \left(\hat{f} - \hat{f}^\prime\right)^2 d\lambda_\theta$. Then, for any $\theta$,
  \begin{align*}
     \frac{1}{m}\sum_{i=1}^m |h(\theta_i) - \hat{h}(\theta_i)|^2
   \le &  \frac{1}{m}\sum_{i=1}^m \int \left|(f - f^\prime)^2 - \left(\hat{f} - \hat{f}^\prime\right)^2\right|^2  d\lambda_{\theta_i}\\
      \le &  \frac{2}{m}\sum_{i=1}^m \int \left((f - f^\prime)^2 + \left(\hat{f} - \hat{f}^\prime\right)^2\right) \left(f - f^\prime - \left(\hat{f} - \hat{f}^\prime\right)\right)^2  d\lambda_{\theta_i}\\
      \precsim &  \epsilon^2 
  \end{align*}

Hence, from the above calculations, $\cN(\epsilon; \cH_r, \fL_2(\hat{\pi}_m)) \le \left(\cN\left( a_3 \epsilon; \cF, \fL_\infty(\hat{\pi}_m)\right)\right)^2$, for some absolute constant $a_3$. 
We also note that,
  \begingroup
  \allowdisplaybreaks
  \begin{align*}
      \operatorname{diam}^2(\cH_r, \fL_2(\lambda_n)) =  \sup_{h, h^\prime \in \cH_r} \|h - h^\prime\|_{\fL_2(\hat{\pi}_m)}^2
      \le & \sup_{h, h^\prime \in \cH_r} \frac{1}{m}\sum_{i=1}^m (h(\theta_i) - h^\prime(\theta_i))^2   \\
      \le & 2  \sup_{h \in \cH_r} \frac{1}{m}\sum_{i=1}^m h^2(\theta_i) \\
      \le & 2 B \sup_{h \in \cH_r} \frac{1}{m}\sum_{i=1}^m h(\theta_i)\\
      \le & 2 B r.
  \end{align*}
  \endgroup
  Hence, $\operatorname{diam}(\cH_r, \fL_2(\hat{\pi}_m)) \le  \sqrt{2 B r} $.
  Thus from \citet[Theorem 5.22]{wainwright_2019}
  \begingroup
  \allowdisplaybreaks
\begin{align}
    \E_{\epsilon} \sup_{h \in \cH_r} \frac{1}{m}\sum_{i=1}^m \sigma_i h(\theta_i)  
   \precsim & \int_0^{\sqrt{2 B r}} \sqrt{\frac{1}{n} \log \cN(\epsilon; \cH_r, \fL_2(\hat{\pi}_m))} d\epsilon \nonumber \\
    \le & \int_0^{\sqrt{2 B r}} \sqrt{\frac{2 \operatorname{Pdim}(\cF)}{m} \log \left(\frac{a_4 m}{\epsilon}\right)} d\epsilon \nonumber \\
    \precsim & \sqrt{2 B r} \sqrt{\frac{\operatorname{Pdim}(\cF) \log m}{m}} + \int_0^{\sqrt{2 B r}} \sqrt{\frac{\operatorname{Pdim}(\cF)}{m} \log (a_2/\epsilon)} d\epsilon \nonumber\\
    \precsim &  \sqrt{\frac{r \log(1/r) \operatorname{Pdim}(\cF) \log m}{m}}. \label{e8}\\
    \le & \sqrt{\frac{(\operatorname{Pdim}(\cF))^2 \log m}{m^2} + r \frac{\operatorname{Pdim}(\cF) \log(m/e\operatorname{Pdim}(\cF)) \log m}{m}}\label{e9}.
\end{align} 
\endgroup
Here, \eqref{e6} follows from Lemma~\ref{lem_17.1}.
\end{proof}
\begin{lemma}\label{lem21}
    Suppose that $v>0$ and $\|f\|_{\fL_\infty([0,1]^d)} \le 2R$. Also assume that $\lambda_\theta \ll \lambda$ and $0 < \frac{d \lambda_\theta}{d\lambda} \le \bar{c}$, almost surely under $\pi$. Then with probability at least $1- 2 \exp\left(- \frac{c_3m v}{ \|\operatorname{KL}(\lambda_{\theta}, \lambda)\|_{\psi_1}}\right)$,
    \begin{align*}
        \|f\|_{\fL_2(\lp)}^2  \le  \frac{3}{2}\|f\|_{\fL_2(\lambda)}^2 + v \quad \text{and} \quad
        \|f\|_{\fL_2(\lambda)}^2  \le 2\|f\|_{\fL_2(\lp)}^2 + 2v.
    \end{align*}
    Here $c_3$ is a constant that depends on $R$ and $\bar{c}$.
\end{lemma}
\begin{proof}
    We first fix $i \in [n]$ and let $ \nu = \frac{1}{2}(\lambda_{\theta_i}+\lambda)$. Then, by Radon-Nykodym theorem, the density of $\lambda_{\theta_i}$ and $\lambda$ w.r.t. $\nu$ exists and is denoted by $p_i = \frac{d\lambda_{\theta_i}}{d \nu}$ and $ p =\frac{d \lambda}{d \nu}$. We let, 
 \begin{align}
     Z_{i} =  \int f^2(\bx) d\lambda_{\theta_i}(\bx) - \int f^2(\bx) d\lambda(\bx)
     = &  \int f^2(\bx) \left(1 - \frac{p(\bx)}{p_i(\bx)}\right) p_i(\bx) d\nu \nonumber\\
     \le & \int f^2(\bx) \left(1 - \frac{p(\bx)}{p_i(\bx)}\right)_+ p_i(\bx) d\nu \label{r1}
     \end{align}
     Further, 
     \begin{align}
    - Z_{i} =  \int f^2(\bx) \left(1 - \frac{p_i(\bx)}{p(\bx)}\right) p(\bx) d\nu 
     \le & \int f^2(\bx) \left(1 - \frac{p_i(\bx)}{p(\bx)}\right)_+ p(\bx) d\nu \label{r2}
     \end{align}
     Suppose $h(\bx) = f^2(\bx)$ and $u = \max\left\{v, \frac{1}{2}\|f\|^2_{\fL_2(\lambda)}\right\}$. Thus, from \eqref{r1} and \eqref{r2},
     \begin{align}
         & Z_i^2 \nonumber\\
         \le & \max \left\{\int f^2(\bx) \left(1 - \frac{p(\bx)}{p_i(\bx)}\right)_+ p_i(\bx) d\nu, \int f^2(\bx) \left(1 - \frac{p_i(\bx)}{p(\bx)}\right)_+ p(\bx) d\nu\right\} \nonumber\\
         \le & \max \left\{\int h^2(\bx) p_i(\bx) d\nu \int \left(1 - \frac{p(\bx)}{p_i(\bx)}\right)_+^2 p_i(\bx) d\nu, \int h^2(\bx) p(\bx) d\nu \int \left(1 - \frac{p_i(\bx)}{p(\bx)}\right)_+^2 p(\bx) d\nu\right\} \nonumber\\
         = & \max \left\{\| h(\bx)\|_{\fL_2(\lambda_{\theta_i})}^2 d_2^2(\lambda, \lambda_{\theta_i}),  \| h(\bx)\|_{\fL_2(\lambda_{\theta})}^2 d_2^2(\lambda_{\theta_i}, \lambda)\right\} \nonumber\\
         \le & 8 R^2 \bar{c} u \max \left\{ d_2^2(\lambda, \lambda_{\theta_i}),  d_2^2(\lambda_{\theta_i}, \lambda)\right\} \nonumber\\
         \le & 8 R^2 \bar{c} u \operatorname{KL}(\lambda_{\theta_i}, \lambda) \nonumber.
     \end{align}
      Then,
     \begin{align}
     \|Z_i\|^2_{\psi_2} & = \|Z_i^2\|_{\psi_1} \le 8 R^2 \bar{c} u  \cdot \|\operatorname{KL}^2(\lambda_{\theta}, \lambda)\|_{\psi_1}  \label{e203}
 \end{align}
 In \eqref{e203}, the first equality follows from \citet[Lemma 2.7.6]{vershynin2018high}. By Hoeffding's inequality,
 \begin{align}
     \prob\left(\left| \sum_{i=1}^m Z_i\right| > m u \right) \le  2 \exp\left(- \frac{cm^2 u^2}{ 8R^2 mu \|\operatorname{KL}(\lambda_{\theta}, \lambda)\|_{\psi_1}}\right) 
     = & 2 \exp\left(- \frac{cm u}{ 8R^2 \bar{c} \|\operatorname{KL}(\lambda_{\theta}, \lambda)\|_{\psi_1}}\right)
         .\label{e202}
 \end{align}
Here, $c_3 = \frac{c}{8R^2\bar{c}}$. Hence, with probability at least, $1- 2 \exp\left(- \frac{c_3m v}{ \|\operatorname{KL}(\lambda_{\theta}, \lambda)\|_{\psi_1}}\right)$,
\begin{align*}
     \|f\|_{\fL_2(\lp)}^2 
    \le \|f\|_{\fL_2(\lambda)}^2 + u
    \le  \|f\|_{\fL_2(\lambda)}^2 + v + \frac{1}{2} \|f\|_{\fL_2(\lambda)}^2
     =  \frac{3}{2}\|f\|_{\fL_2(\lambda)}^2 + v. 
\end{align*}
Furthermore,
\begingroup
\allowdisplaybreaks
\begin{align*}
     \|f\|_{\fL_2(\lambda)}^2 
    \le \|f\|_{\fL_2(\lp)}^2 + u
    \le  \|f\|_{\fL_2(\lp)}^2 + v + \frac{1}{2} \|f\|_{\fL_2(\lambda)}^2 \implies  \|f\|_{\fL_2(\lambda)}^2 \le   2\|f\|_{\fL_2(\lp)}^2 + 2v.
\end{align*}
\endgroup
\end{proof}
\begin{lemma}\label{lem_17.1}
   For any $\delta \le 1/e$, $\int_0^\delta \sqrt{\log(1/\epsilon)} d\epsilon \le 2 \delta \sqrt{\log(1/\delta)}$.
\end{lemma}
\begin{proof}
    We start by making a transformation $x = \log(1/\epsilon)$ and observe that,
    \begingroup
    \allowdisplaybreaks
    \begin{align}
         \int_0^\delta \sqrt{\log(1/\epsilon)} d\epsilon 
        = \int_{\log(1/\delta)}^\infty \sqrt{x} e^{-x} dx
        = & \int_{\log(1/\delta)}^\infty \sqrt{x} e^{-x/2} e^{-x/2} dx \nonumber \\
        \le & \sqrt{\log(1/\delta)} e^{-\frac{1}{2}\log(1/\delta)} \int_{\log(1/\delta)}^\infty e^{-x/2} dx \label{e112}\\
        = & 2 \delta \sqrt{\log(1/\delta)}. \nonumber
    \end{align}
    \endgroup
    In the above calculations, \eqref{e112} follows from the fact that the function $\sqrt{x} e^{-x/2}$ is decreasing when $x\ge 1$.
\end{proof}
\begin{lemma}\label{lem:8}
    For any $x,y >0$, $x \log x \le y + x \log(1/ye)$.
\end{lemma}
\begin{proof}
    Let $f(r) = r \log(1/r)$. Then, $f^\prime(r) = - \log(re)$ and $f^{\prime \prime}(r) = -1/r$. Thus, $f(\cdot)$ is concave and thus, for any $x,y>0$,
    \begin{align*}
        f(x) \le  f(y) + f^\prime(y) (x-y) 
        = &  -y \log y - (\log y +1) (x-y) \\
        = & y - x \log (ye) \\
        = & y + x \log(1/ye).
    \end{align*}
\end{proof}
\begin{lemma}\label{app_bern}
    Suppose that $g(\cdot)$ be a non-negative real-valued function such that $B = \|g\|_\infty < \infty$. Let $Z_1,\dots, Z_n$ be independent random variables. Then, with probability at least $1- e^{-nt}$,
     \[\frac{1}{n}\sum_{i=1}^n g(Z_i) \le \frac{2}{n}\sum_{i=1}^n \E g(Z_i) + \frac{7Bt}{3} .\]
\end{lemma}
\begin{proof}
    Let $Y_i = g(Z_i) - \E g(Z_i)$. Also let $v >0$ and $u = \max \left\{v, \frac{1}{n} \sum_{i=1}^n \E g(Z_i)\right\}$. Clearly,
    \begin{align}
        \E Y_i^2 = \operatorname{Var}(Y_i) = \operatorname{Var}(g(Z_i)) \le \E g^2(Z_i) \le B \E g(Z_i).
    \end{align}
   Thus, $\sigma^2 = \sum_{i=1}^n \E Y_i^2 \le B \sum_{i=1}^n \E g(Z_i) \le nB u$. From Lemma~\ref{bernstein}, we observe that,
   \begin{align*}
       \prob \left(\left|\sum_{i=1}^n g(Z_i) - \E g(Z_i)\right| > n u \right) 
       \le &  \exp\left(-\frac{n^2 u^2 }{2 \sigma^2 + nBu/3}\right) \\
       \le & \exp\left(-\frac{3 n u}{7 B}\right) \le \exp\left(-\frac{3 n v}{7 B}\right)
   \end{align*}
  Thus, with probability at least $1-\exp\left(-\frac{3 n v}{7 B}\right)$,
  \[\frac{1}{n}\sum_{i=1}^n g(Z_i) \le \frac{1}{n}\sum_{i=1}^n \E g(Z_i) + u \le \frac{2}{n}\sum_{i=1}^n \E g(Z_i) + v .\]
  Taking $v = 7Bt/3$, we get the desired result.
\end{proof}
\bibliographystyle{apalike}

\end{document}